%% file: paper.tex
\documentclass[natbib,smallextended]{svjour3}

\usepackage{epsfig}
\usepackage{graphicx}
\usepackage[tight,footnotesize]{subfigure}
\usepackage{caption}
\usepackage{float}
\usepackage{wrapfig}
\usepackage{verbatim}
\usepackage{xcolor}
\usepackage{nomencl}
\usepackage{url}
\usepackage[pagebackref=true,breaklinks=true,letterpaper=true,bookmarks=false]{hyperref}
\usepackage{type1cm}
\usepackage{macrosfabbri}
\usepackage{datetime} % for \currenttime
\usepackage{setspace} % for changing spacing around certain eqs/align

\renewcommand{\cite}{\citep}
\renewcommand{\textwidth}{6in}

\journalname{International Journal of Computer Vision}

\newcommand{\draftnote}[1]{}
\newcommand{\indraftnote}[1]{}

\numberwithin{equation}{section}

\begin{document}
\input{front-matter}

\section{Introduction}

The automated estimation of camera parameters and the 3D reconstruction from
multiple views is a fundamental problem in computer vision.  These tasks rely
on correspondence of image structure commonly in the form of keypoints,
although dense patches and curves have also been used.
\begin{figure}[h!]
%  \vspace{-19.20em}
%  \vspace{-5em}
  \centering
   %/mnt/cortex-rfabbri/mvg/presents/multiview-vision-notes/figs/sift-vs-curves/mazda-img1-gray.png
   %/home/rfabbri/mvg/doc/Fabbri-PhD-2010/figs/sift-fail/candle-plastic-image-1-big.png
   % /home/rfabbri/work/sp/images/landing-ss-1.png
   % /vision/video/rfabbri/nonrigid/snakes/snakepit/snakepit-10.png
   \includegraphics[width=\linewidth]{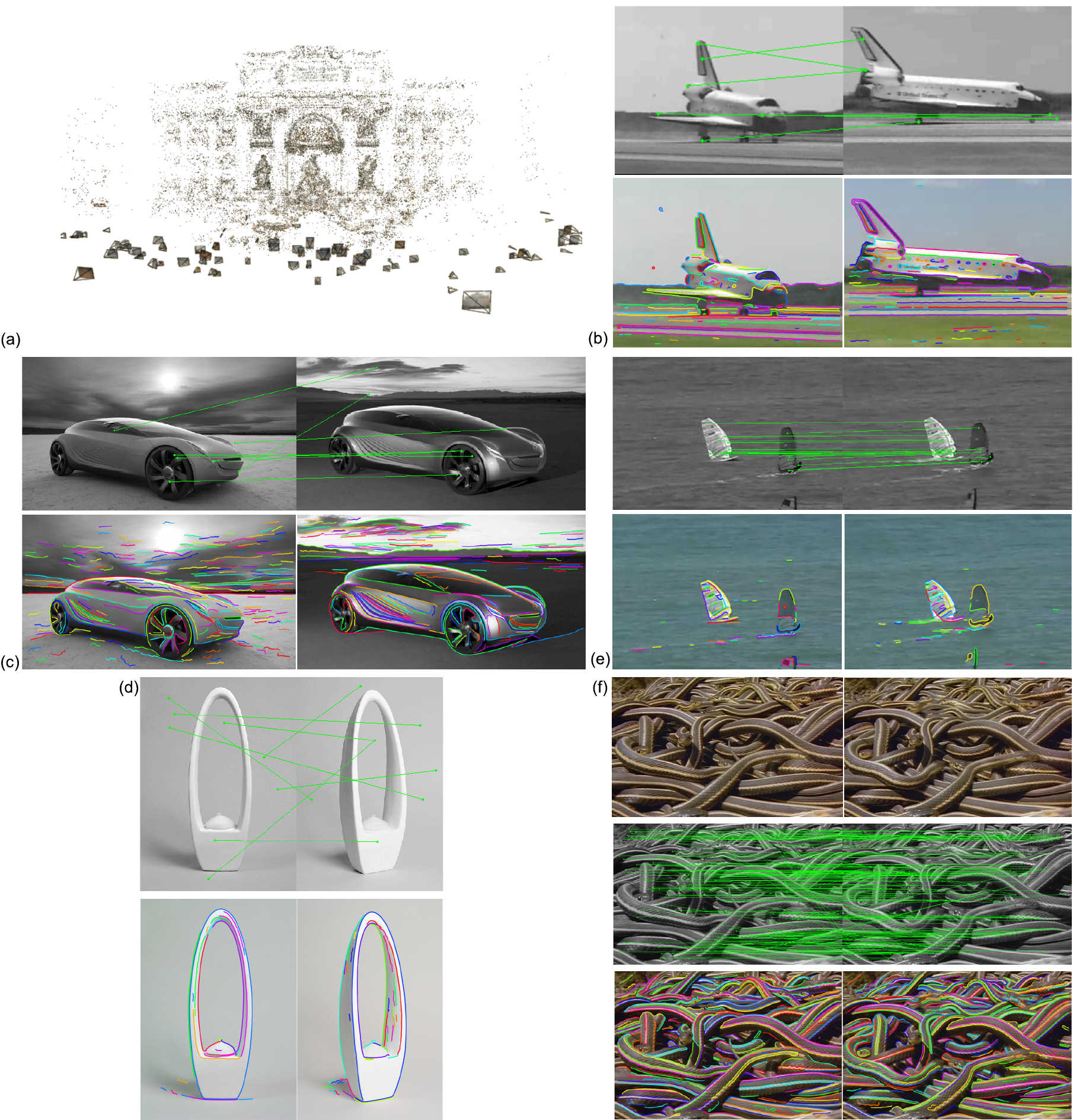}
   \caption{
   \small
   (a) Keypoint-based approaches give a sparse
   point cloud reconstruction~\cite{Argarwal:Snavely:etal:ICCV09,Heinly:Frahm:etal:CVPR2015} which
   can be made richer, sharper, and denser with curve structure. (b) Wide
   baseline views not sharing interest points often share curve structure. (c,d) Not
   enough points matching views of homogenous
   objects with
   sufficient curve structure. (e)~Each moving object or non-rigid object requires its own set of
   features, but may not have enough texture.
   }
   \vspace{-2em}
\label{fig:jul04:xx1}
\end{figure}
\textbf{Keypoint-based methods} extract point features designed
be stable with view variations. 
They satisfy certain local conditions in the spatial and
scale dimensions~\cite{mikolajczyk:schmid:IJCV04,Harris:Stephens:Edge:Corner,Moravec:ICAI77,
lowe:distinctive:IJCV04}, and are attributed with a local description of
the image~\cite{Mikolajczyk:Schmid:PAMI05}. While many of these 
points are not stable with view changes in that they disappear/appear or vary
abruptly, many are stable enough for matching with
\ransac~\cite{Fischler:Bolles:RANSAC:CACM81} to
initialize camera models for bundle
adjustment~\cite{Hartley:Zisserman:multiple:view,Pollefeys:VanGool:etal:handheld:IJCV2004,Argarwal:Snavely:etal:ICCV09,Heinly:Frahm:etal:CVPR2015,Diskin:Vijayan:JEI2015}.

A major drawback of interest points is their sparsity compared to the
curves and surfaces composing the scene, producing a cloud of 3D points
where geometric structure is not
explicit, Figure~\ref{fig:jul04:xx1}(a).  This is less of a problem for
camera estimation~\cite{Argarwal:Snavely:etal:ICCV09}, but in applications
such as 
%registration of new views (without resorting to the originating 2D features), 
architecture, industrial design, object recognition, and
robotic manipulation, explicit 3D geometry is required. Moreover, meshing
unstructured point
clouds produces oversmoothing~\cite{Kazhdan:etal:SGP06,Furukawa:Ponce:PAMI2010}.
% xxx uncomment for thesis.
%Fourth, matching of interest points breaks down at regions with considerable
%surface curvature and foreshortening, such as parts of surfaces curving away
%from the viewer. This is due to the fact that the change from one view to the
%other can no longer be modeled by a simple transformation as assumed by point
%features (in the case of \sift, this is a rotation and scaling), but the
%transformation needs to account for surface curvature.
%The method we propose here produces a 3D curve sketch which can be used as
%scaffolding to stretch a surface on.
These techniques are therefore inadequate for
man-made environments~\cite{Simoes:etal:SVR2014} and objects such as
cars~\cite{Shinozuka:Saito:VRIC14}, non-Lambertian surfaces such as that of the
sea, appearance variation due to changing
weather~\cite{Baatz:Pollefeys:etal:ECCV12}, and wide
baseline~\cite{Moreels:Perona:IJCV07}, Figure~\ref{fig:jul04:xx1}(b).
We claim that by using even \emph{some} curve information the 3D
reconstruction can be made more structually rich, sharper and less sparse. 

The use of keypoints requires an
abundance of features surviving the variations
between views. While this occurs for many scenes, in many others this is not the case, such as
$(i)$ homogeneous regions from man-made objects,
Figure~\ref{fig:jul04:xx1}(c,d); $(ii)$ moving objects require their
own set of features which are too few without sufficient
texture, Figure~\ref{fig:jul04:xx1}(e); $(iii)$ non-rigid objects require a
rich set of features per local patch,
Figure~\ref{fig:jul04:xx1}(f).
While curve features, like keypoints,
may not be abundant the interior of highly homogeneous regions or
in nonrigid or small moving objects,
\emph{curves rely on texture to a much lesser extent}, as shown in Figure~\ref{fig:jul04:xx1}(c,d,e);
the boundary of homogeneous regions are good cues (often the
\emph{only} cues), enabling extracting information structure and motion.
In all these situations, there may be sufficient curvilinear structure,
Figure~\ref{fig:jul04:xx1}, motivating augmenting the use of interest points
with curves.

\textbf{Pixel-based multiview stereo} 
relies on matching
intensities across views, resulting in a dense point cloud
or a mesh
reconstruction. 
This produces
detailed 3D reconstructions of objects imaged under controlled conditions by a
large number of precisely calibrated cameras~\cite{Furukawa:Ponce:CVPR2007,
Habbecke:Kobbelt:CVPR2007,Hernandez:Schmitt:CVIU04,Goesele:etal:ICCV07,Seitz:etal:CVPR06,Calakli:etal:3DIMPVT2012,Restrepo:etal:JPRS2014}.  However, 
there are a number of limitations: they typically assume that the
scene consists of a single object or that objects are of a specific type, such
as a building; they often require accurate camera calibration and operate
under controlled acquisition; and they need to be initialized by the visual hull
of the object or a bounded 3D voxel volume, compromising aplicability for
general scenery.

%Fourth,
%volumetric methods require a large amount of memory.

%\indraftnote{Ric: what is the purpose of the following paragraph? Ben: it is to
%review 2-view stereo; remember, this paper is about general use of curves, not
%just for a large number of views, but for two view case as well.}
%\indraftnote{Ric: Need to say more of the drawbacks/advantages in the following
%paragraph. Ben: looks clear to me - the drawbacks are short baseline and use of
%heuristics.}

%A large body of literature in \textbf{two-view short-baseline stereo} has been
%developed based on correlating unorganized and often sparse feature points by
%matching some aspects of the local region surrounding the
%feature~\cite{Marr:Poggio:Stereo,Grimson:Stereo:1981,Pollard:PMF:Stereo}
%to disambiguate correspondences. A number of criteria such as smoothness,
%uniqueness, ordering, limited disparity, and limited orientation disparity, have
%been used to deal with the inherent ambiguity. However, these can break down,
%especially with wide baseline, with multiple nearby structures, or when
%discontinuities and branching structures exist~\cite{Li:Zucker:2003}.

\textbf{Curve-based multiview methods} 
can be divided into three categories: $(i)$
convex hull construction, $(ii)$ occluding contour reconstruction, 
and $(iii)$ use of differential geometry in binocular and trinocular stereo.
First, when many views are
available around an object, a visual hull has been constructed
from silhouette curves and
then evolved to optimize
photometric constraints while constraining the
projection to the silhouettes. The drawbacks
are similar to those of pixel-based multiview stereo.
Second, the occluding contours
extracted from frames of a video have been used to reconstruct a local surface model
given the camera parameters
for each frame. These methods require highly controlled acquisition and
image curves that are easy to segment and track.
In addition, since only silhouettes are used, internal surface variations
which may not map to apparent contours in any view will not be captured, \eg,
surface folds of a sculpture. 

Third, some methods employ curve differential geometry in correlating structure
across views. Complete 3D reconstruction pipelines based on straight
lines~\cite{Lebeda:etal:ACCV2014,Zhang:line:PHDThesis2013,Fathi:etal:AEI2015},
algebraic and general curve features~\cite{Teney:Piater:3DIMPVT12,Litvinov:etal:IC3D2012,Fabbri:Kimia:CVPR10,Fabbri:Kimia:Giblin:ECCV12,Wendel:etal:CVWW2011}
have been proposed. The compact curve-based 3D representation that has found
demand in seveal tasks: fast recognition of general 3D
scenery~\cite{Wendel:etal:CVWW2011}, efficient transmission of general 3D
scenes, scene understanding and modeling by reasoning at
junctions~\cite{Mattingly:etal:JVLC2015}, consistent non-photorealistic
rendering from video~\cite{Chen:Klette:IVT2014}, modeling of branching structures, to name a
few~\cite{Rao:etal:IROS2012,Kowdle:etal:ECCV10,Ruizhe:Medioni:CVPR2014}.

\textbf{Related work:} Differential geometry does not provide hard constraints for matching
in static binocular stereo, as known for tangents and
curvatures~\cite{Robert:Faugeras:1991}, and shown here for higher order.
Heuristics have been employed in short baseline to limit orientation
difference~\cite{Arnold:Binford:1978,Grimson:Stereo:1981,Sherman:Peleg:PAMI1990},
to match appearance via locally planar
approximations~\cite{Schmid:Zisserman:2000},
or to require 3D curve reconstructions arising from two putative
correspondence pairs to have minumum torsion~\cite{Li:Zucker:2003}. 
When each stereo camera provides a video and the scene (or stereo head) moves
rigidly, differential geometry provides a hard
constraint~\cite{Faugeras:Papadopoulo:IJCV93,Papadopoulo:PhD:96,Faugeras:Papadopoulo:Chapter1992}.
Differential geometry is
more directly useful in trinocular and multiview stereo, as pioneered
by~\citet{Ayache:Lustman:1987}, due to the constraint
that corresponding pairs of points and tangents from two views
uniquely determine a point and tangent in a third to match line segments
obtained from edge linking~\cite{Ayache:Lustman:1987,Spetsakis:IJCV1990,Shahsua:Trilinearity:ECCV1994,Hartley:ICCV1995}.
\citet{Robert:Faugeras:1991} extended this
to include curvature: 3D curvature and normal can be reconstructed from 2D curvatures at two views,
determining the curvature at a third. The use of curvature
improved reconstruction precision and density,
with heuristics such as the ordering constraint~\cite{Kanade:Ordering:1985}.  
\citet{Schmid:Zisserman:2000} 
derived multiview curvature transfer when only the trifocal tensor
is available, by a projective-geometric approach to the osculating circle as a
conic.

%Drawback of these methods is the dependence on
%precise camera parameters, and biased criteria in the binocular case.
% the above drawback doesn't apply to the Schmid paper, does it?

Curves have also been employed for camera estimation using
the concept of epipolar tangencies: corresponding epipolar lines
are tangent to a curve at corresponding
points~\cite{Giblin:Motion:Book,Astrom:Cipolla:Giblin:IJCV1999,Astrom:Kahl:PAMI99,Kahl:Heyden:ICCV98,
Porrill:Pollard:1991,Kaminski:Shashua:2004,Berthilsson:etal:IJCV2001,Wong:etal:IWVF01,Mendonca:etal:PAMI2001,Wong:Cipolla:IP04,Furukawa:etal:PAMI06,Hernandez:etal:PAMI07,Cipolla:etal:ICCV95,Reyes:Corrochano:IVC05,Sinha:etal:CVPR04}.
This is used to capture epipolar geometry or relative pose.

\textbf{Curve-Based Multiview Geometry:} 
What would be desirable is a generally applicable
framework, \eg, a handheld video acquiring images around objects or a set of
cameras monitoring a scene, where image curve structure can
be used to estimate camera parameters and reconstruct a 3D curve sketch on which a surface can be
tautly stretched like a tent on a metallic scaffold. This paper provides
the \emph{mathematical foundation} for this curve-based approach. Image curve
fragments are attractive because they have good localization,
have greater invariance than interest points to
changes in illumination, are stable over a greater range of baselines, and are
denser than interest points. Moreover, for the special case of occluding
contours, dense 3D surface patch reconstructions are available. The notion that
image curves contain much of the image information is supported by recent
studies~\cite{Koenderink:Wagemans:etal:iPerception2013,Zucker:PIEEE2014,Kunsberg:Zucker:LNM2014,Cole:etal:SIGGRAPH09}.

%Furthermore,
%stationary curves such as reflectance or ridge curves provide boundary
%conditions for
%surface reconstruction, while occluding contour variation across views 
%directly leads to surface reconstruction~\cite{Giblin:Motion:Book}.

%The use of image curves as the structure to be correlated across images
%and as a path to general scene reconstruction is not without challenges.
%The process of linking edges into curve fragments is fraught with
%ambiguities so that there is instability with view
%variation. Also, even if it is known that two curve fragments correspond, there
%remains an intra-curve correspondence ambiguity. This motivates the use of small
%curve segments or points with differential geometry.

This paper develops the theoretical foundations for using the
differential geometry of image curve structure as a complementary alternative
to interest
points. This paper is organized along the lines of these questions:
$(i)$ How does the differential geometry of a space curve map to the
differential geometry of the image curve it projects to? $(ii)$ How can
the differential geometry of a space curve be reconstructed from that of 
two corresponding image curves? $(iii)$ How does the differential geometry of an image curve evolve
under camera motion?
Section~\ref{sec:notation} establishes notation for
image and space curves, camera projection and
motion, and discusses the distinction between \emph{stationary} and
\emph{non-stationary} 3D contours.
Section~\ref{sec:single:view:projection} 
relates the differential geometry of image curves, \ie, tangent,
curvature, and curvature derivative from the differential geometry of the space
curves they arise from, \ie, tangent and normal, curvature, torsion, and
curvature derivatives. Section~\ref{sec:reconstr} derives the differential
geometry of a space curve at a point from that at two corresponding image
curve points showing the key result that the ratio of parametrization speeds
is an intrinsic quantity. The key new result 
is the reconstruction of torsion and curvature derivative, given
corresponding differential geometry in two views.
Section~\ref{sec:projection:differential:motion} considers differential camera
motion and relates the differential geometry of a space curve to that
of the image and camera motion.  Results are provided concerning
image velocities and accelerations with respect to time 
for different types of curves; in particular, distinguishing apparent and stationary contours
requires second-order time
derivatives~\cite{Giblin:Motion:Book}. We
study the spatial variation of the image velocity field along curves,
which can be useful for exploiting neighborhood consistency of
velocity fields along curves.  The main new result generalizes
a fundamental curve-based differential structure from motion
equation~\cite{Papadopoulo:Faugeras:ECCV96,Papadopoulo:PhD:96} to occluding
contours.

\emph{This paper integrates the above results under the umbrella
of a unified formulation and completes missing relationships.} As a
generalized framework, it is expected to serve as reference for
research relating local properties
of general curves and surfaces to those of cameras and images.
Much of this has already been done, but a considerable amount has not,
as mentioned earlier, and most results are
scattered in the literature.
This theoretical paper has been the foundation of practical work already
reported on reconstruction and camera estimation systems as follows.
First, the pipeline for the reconstruction of a 3D Curve Sketch from image
fragments in numerous
views~\cite{Fabbri:Kimia:CVPR10,Fabbri:Kimia:CVPR16} relies on the results on 3D reconstruction and
projection of differential geometry reported in this paper, and a future
extension of this pipeline would require most results in this
paper. Second, a recent practical algorithm for pose estimation based on differential
geometry of curves~\cite{Fabbri:Kimia:Giblin:ECCV12}
relies on
the theory reported in this paper, treating the camera pose as unknowns and
using differential geometry to solve for them, \cf\ ensuing
efforts by~\citet{Kuang:Astrom:ICCV2013,Kuang:Astrom:etal:ICPR2014}. Third, work on
differential camera motion estimation from families of curves based on
the present work has been explored
by~\citet{Jain:PHD:2009,Jain:etal:Tracking:CVIU07,Jain:etal:ICIP2007}.  These
works are currently under intense development in order to build a complete
structure from motion pipeline
based on curves, which would use the majority of the results described in
this paper, including analogous results for multiview surface differential geometry that are
under development.
% perhaps include pipeline fig
%xxx perhaps for conclusion:
%\noindent \textbf{Stereo matching:}
%\begin{itemize}
%\item In two views, one can employ an idea similar to~\cite{Li:Zucker:2003},
%where the compatibility of \emph{two} neighboring
%point-tangent-curvature matches are assessed according to a cost dependent on
%space curve torsion, which is then
%minimized through relaxation labeling. Using the torsion reconstruction formulas
%in this paper, one can in principle consider torsion fully in the compatibility
%Taylor expansions, so that the minimization doesn't produce a bias towards
%planar curves.
%\item In three or more views, we can employ curvature derivative and torsion
%reconstruction from two views, reprojected onto a third view as curvature
%derivative for confirmation. Although curvature derivative is
%numerically difficult to compute, we believe even a qualitative or coarse measure can be
%of help.
%\end{itemize}
%\noindent \textbf{Calibration from contours:}
%\begin{itemize}
%\item Papadopoulo's equation~\cite{Papadopoulo:PhD:96} provides a way to recover calibration (pose) of a video
%sequence from curves tracked along their normal direction, without any prior
%knowledge of epipolar geometry. The numerical methods for this are known to be
%challenging and their improvement is an area of research. In this paper we
%corrected and extended this equation to also include occluding contours, opening
%new possibilities for camera calibration from tracked contours.
%\end{itemize}

\section{Notation and Formulation}\label{sec:notation}

\subsection{Differential Geometry of Curves}
For our purposes, a 3D space curve $\Curve$ is a smooth map $S \mapsto
\Gama^w(S)$ of class $C^\infty$ from an interval
of $\mathbb{R}$
to $\mathbb{R}^3$, where $S$ is an arbitrary parameter,
$\tilde S$ is the arc-length parameter, and the superscript
$w$ denotes the world coordinates.
The local Frenet frame of $\Curve$ in world
coordinates is defined by the unit vectors tangent $\T^w$,
normal $\N^w$, binormal $\B^w$; $G$ is speed of parametrization, 
curvature $K$, and
torsion $\tau$. 
Similarly, a 2D curve $\gamma$ is a map $s \mapsto \gama(s)$ of class $C^\infty$ from
an interval of $\mathbb{R}$ to $\mathbb{R}^2$, where $s$ is an arbitrary parameter, $\tilde s$
is arc-length, $g$ is speed of parametrization,
$\t$ is (unit) tangent, $\n$ is (unit) normal, $\kappa$ is curvature, and $\kappa'$ is curvature derivative.
We will be concerned with regular curves, so that $G\neq 0$ and $g\neq 0$ unless
otherwise stated.
By classical differential geometry~\cite{Carmo:Diff:Geom}, we have
\begin{equation}\label{eq:frenet:explicit}
\left\{
\begin{aligned}
G  &= \|\Gama^{w'}\|  \\
\T^w &= \frac{\Gama^{w'}}{G}  &\hspace{3mm}
\N^w &= \frac{\T^{w'}}{\|\T^{w'}\|}  &\hspace{3mm}
\B^w &= \T^w \times \N^w \\
K  &= \frac{\|\T^{w'}\|}{G}  &
\dot{K} &= \frac{K'}{G}  &
\tau &= \frac{-\B^{w'}\cdot\N^w}{G},
\end{aligned}\right.
\ \ \ 
\begin{bmatrix}
\T^{w'}\\
\N^{w'}\\
\B^{w'}
\end{bmatrix} = 
G
\begin{bmatrix}
\hfill 0\,   & \hfill K &\,\,\, 0\ \  \\
-K & \hfill 0 \, &\,\,\, \tau\ \ \\
\hfill 0\, & -\tau\, &\,\,\, 0\ \ 
\end{bmatrix}
\begin{bmatrix}
\T^w\\
\N^w\\
\B^w
\end{bmatrix},%\label{eq:frenet}
\end{equation} 
and
\begin{equation}\label{eq:frenet2D:explicit}
\begin{aligned}
g  &= \|\gama'\|, &\hspace{7mm}
\t &= \frac{\gama'}{g}, &\hspace{7mm}
\n &= \t^\perp, &\hspace{7mm}
\kappa &= \frac{\t'\cdot \n}{g}, &\hspace{7mm}
\dot \kappa &= \frac{\kappa'}{g},
\end{aligned}
\end{equation} 
%\begin{equation}\label{eq:frenet2D:explicit}
%\hspace{-4.4in}\left\{
%\begin{aligned}
%g  &= \|\gama'\|\\
%\t &= \frac{\gama'}{g} & \hspace{3mm}
%\n &= \t^\perp\\
%\kappa &= \frac{\t'\cdot \n}{g} &
%\dot \kappa &= \frac{\kappa'}{g}
%\end{aligned}\right.
%\end{equation} 
where prime ``$\,'$'' denotes
differentiation with respect to an arbitrary spatial parameter ($S$ or $s$).
We use dot ``$\dot\ $'' 
to denote differentiation with respect to arc-length ($\tilde S$ or $\tilde s$)
only when an entity clearly belongs to either a space or an image
curve.
The matrix equations on the right of~\eqref{eq:frenet:explicit} are the Frenet equations. 
Note that both the curvature
derivatives $\dot{K}$ and $\dot \kappa$ are intrinsic quantities.

\begin{comment} % xxx uncomment for thesis?
\noindent The (geometric) Taylor expansion of $\Gama(s)$ for an arbitrary parameter $S$ is
\begin{align}\label{eq:taylor:s:grouped}
\Gama^w(S) = &\Gama^w_0 + S\,G_0\T_0^w + 
%
\frac{1}{2}S^2\left[G_0'\T_0^w + G_0^2K_0\N_0^w \right] +\\
%
&\frac{1}{6}S^3\left[ (G_0'' - G_0^3K_0^2)\T_0^w + (3G_0G_0'K_0 +
G_0^3\dot{K_0})\N_0^w +
G_0^3K_0\tau_0\B_0^w\right]+ \boldsymbol{O}(S^4),\notag
\end{align}
where the subscript 0 indicates evaluation at $S = 0$.
For the first order geometry, we have
\begin{align}
\T^w(S) &= \T_0^w + S\T_0^{w'} + \frac{S^2}{2}\T_0^{w''} + \boldsymbol{O}(S^3)\notag\\
      &= \T_0^w + S\,G_0K_0\N_0^w + \frac{S^2}{2}\, \left[(G_0'K_0 +
      G_0^2\dot{K_0})\N^w_0 -
      G^2_0 K_0^2\T_0^w +
      G^2_0K_0\tau_0\B^w_0\right] + \boldsymbol{O}(S^3) .\notag
\end{align}
Similarly, for second order geometry,
\begin{empheq}[left=\empheqlbrace]{align}
\N^w(S) &= \N_0^w + S\,G(-K\T^w + \tau\B^w) + \boldsymbol{O}(S^2)\\
K(S) &= K_0 + SG_0\dot{K}_0 + O(S^2)\\
\B^w(S) &= \T^w(S)\times\N^w(S)
\end{empheq}
and for third order geometry
\begin{align}
\tau(S) = \tau_0 + O(S).
\end{align}

\noindent All formulas in
Equation~\eqref{eq:frenet:explicit} apply also to $\gama$ by
setting $\tau = 0$. 
\end{comment}

\subsection{Perspective Projection}
\begin{figure}
\centering
   \subfigure[]{ %
      \label{fig:1view}
      \includegraphics[height=2in]{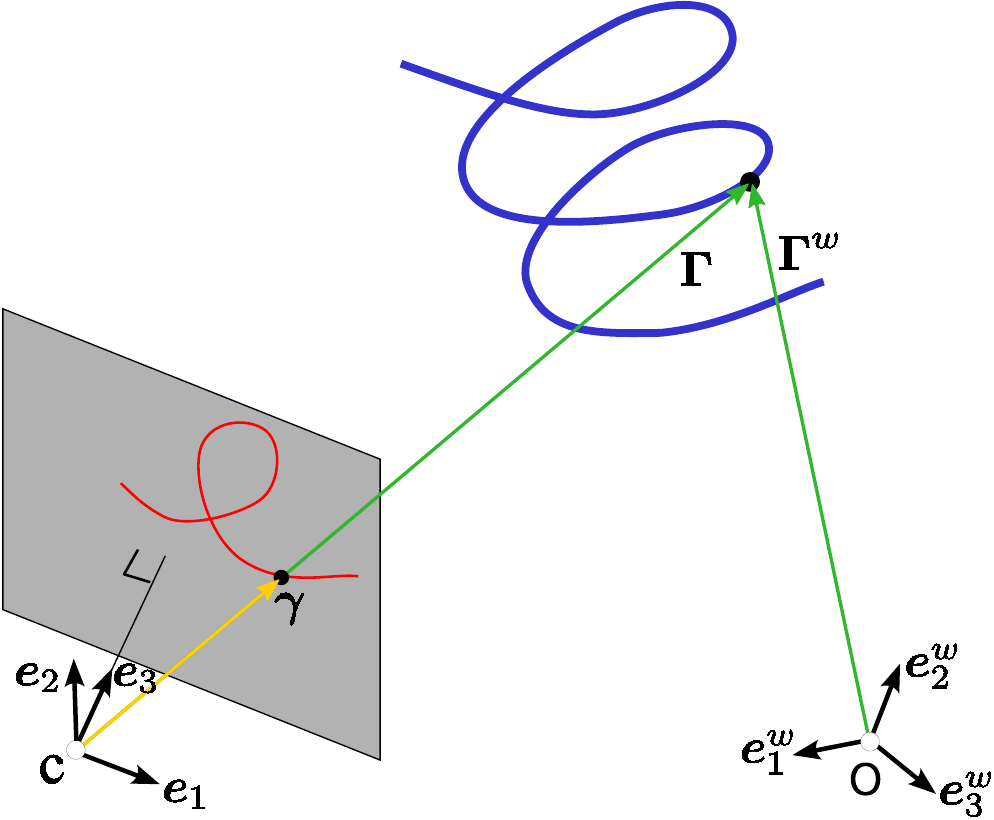}
    }
   \subfigure[]{ %
      \label{fig:mview:rig}
      \includegraphics[height=2.5in]{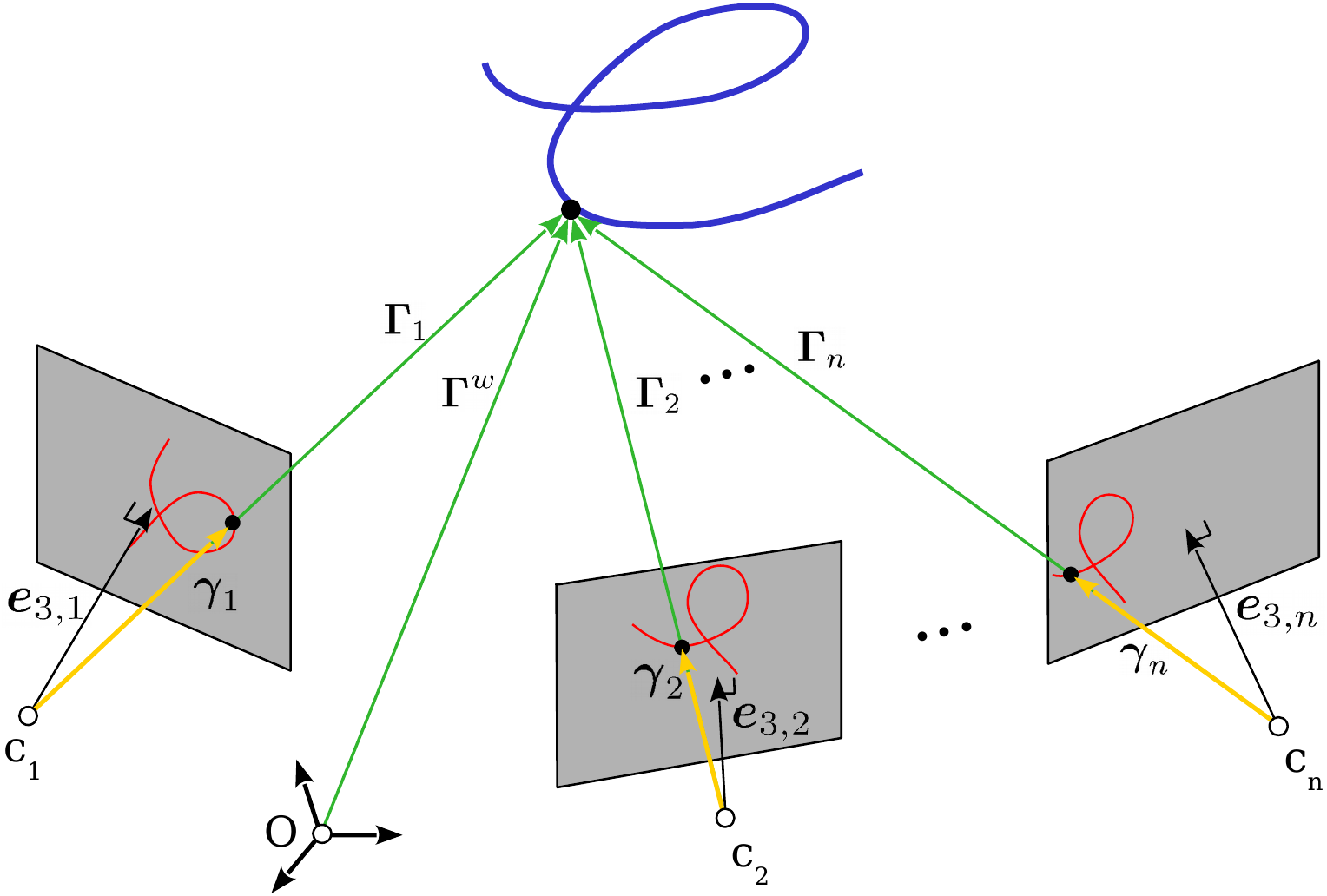}
    }
\caption{% 
The perspective projection of a space curve in (a) one view, and
(b) $n$ views.
}
\end{figure}
The projection of a 3D space curve $\Gamma$ into a 2D image curve $\gamma$ is
illustrated by Figure~\ref{fig:1view},
where the world coordinate system is centered at $O$ with basis
vectors $\{\e_1^w, \e_2^w, \e_3^w\}$. The
\emph{camera coordinate system} is centered at $\bc$
with basis vectors $\{\e_1, \e_2, \e_3\}$. A generic way of referring to
individual coordinates is by means of the specific subscripts $x,y$ and $z$ attached
to a symbol, \ie, $\boldsymbol{v} = [v_x, v_y, v_z]^\top$ for any vector
$\boldsymbol v$; other subscripts denote partial differentiation.
When describing coordinates in the
camera coordinate system we drop the $w$ superscript, \eg, $\Gama$ versus
$\Gama^w$, which are related by
\begin{equation}\label{eq:coord:transf:RT}
\Gama = \rot (\Gama^w - \bc) = \rot\Gama^w + \transl,
\end{equation}
where $\rot$ is a rotation and $\transl = -\rot\bc$
denotes the world coordinate origin in the camera coordinate system.
\nomenclature[04000]{$\transl(t)$}{
$(meters)$
The translation vector that is added when
passing from coordinates at time~$t$ to the world coordinates
}%

The projection of a 3D point $\Gama = [x,\, y,\, z]^\top$ onto the image plane
at $z=1$ is the point $\gama = [\uu,\,\vv,\,1]^\top$ related by
\begin{equation}\label{eq:projection} \Gama =
\depth\gama\,\,\,\,\text{or}\,\,\,\, [x,\, y,\, z]^\top =
[\depth\uu,\,\depth\vv,\,\depth]^\top,
\end{equation}
where we say that $\gama$ is in normalized image coordinates (focal distance is
normalized to 1), and the depth is $\depth = z = \e_3^\top\Gama$%
  \nomenclature[12000]{$\depth(\xi,\eta,t)$}{$(meters)$ Depth at pixel
  $(\xi,\eta)$ measured along $z$ axis of camera at time $t$: $\Gama =
  \depth\gama$}
  \nomenclature[00800]{$\uu,\vv$}{$(pixels)$ The horizontal and vertical coordinates of the image,
  respectively.}%
  \nomenclature[10100]{$\gama(\xi,\eta,t)$}{
  (vector, $meters$) 
  The vector from the camera center to
  image point $(\xi,\eta)$ at time $t$, in the coordinates of the camera at time~$t$.
  }%
  \nomenclature[10300]{$\gama^w(\xi,\eta,t)$}{
  (vector, $meters$) 
  $\gama(\xi,\eta,t)$ in world coordinates
  }%
from the third coordinate equation. Observe that image points are treated as 3D points
with $z = 1$. Thus, we can write
\begin{align}
\gama &= \frac{\Gama}{\depth}.
\label{eq:projection:isolated:gamma}
\end{align}
We note that $\f^\top\gama^{(i)} = 0$ and $\f^\top\Gama^{(i)} = \depth^{(i)}$,
where $\gama^{(i)}$ is the $i^{th}$ derivative of $\gama$ with respect to an
arbitrary parameter, for any positive integer $i$. Specifically, 
\begin{equation}\label{eq:depth:derivs}
\depth = z,\qquad \depth' = G T_z,\qquad \depth'' = G'T_z +
G^2K N_z.
\end{equation}
It is interesting to note that at near/far points of the curve, \ie, $\depth' =
0$, $T_z = 0$.

In practice, normalized image coordinates $\gama = 
[\uu,\,\vv,\,1]^\top $ are described in terms of image
pixel coordinates $\gama_{im} = 
[x_{im},\, y_{im},\, 1]^\top$ through the 
intrinsic parameter matrix $\KK_{im}$ according to
\begin{equation}\label{eq:intrinsic:parameter:transf}
\gama_{im} = \KK_{im}\gama,
\ \ \ \ \ \
\KK_{im} = \begin{bmatrix}
\alpha_\uu & \sigma & \uu_o\\
0 &\alpha_\vv &  \vv_o\\
0 & 0 &  1
\end{bmatrix},
\end{equation}
where as usual $\uu_o$ and $\vv_o$ are the principal points, $\sigma$ is skew, and
$\alpha_\uu$ and $\alpha_\vv$ are given by the focal length divided by the width and
height of a pixel in world units, respectively.

%-------------------------------------------------------------

\subsection{Discrete and Continuous Sets of Views}
Two scenarios are considered. The first scenario consists of a
\emph{discrete set of views}
where a set of $n$ pinhole cameras observe a scene as shown in
Figure~\ref{fig:mview:rig}, with the last subscript in the symbols indentifying
the camera, \eg, $\gama_i$ denotes an image point in the $i^{th}$ camera,
and $\e_{3,i}$ denotes $\e_3$ in the $i^{th}$ view. The second scenario consists
of a \emph{continuous
set of views} from a continuously moving camera observing a space curve which may
itself be moving, $\Gama^w(S,t) = [x^w(S,t),\, y^w(S,t),\, z^w(S,t)]^\top$, 
where $S$ is the parameter along the curve and $t$ is time, 
described in the camera coordinate system associated with time $t$ as $\Gama(S,t) = \left[ x(S,t),\,
y(S,t),\, z(S,t) \right]^\top$, Figure~\ref{fig:mview:rig:continuous}. 
For simplicity, we often omit the parameters $S$ or $t$.
\begin{figure}
\centering
%  \scalebox{0.7}{\includegraphics{multiview-rig-motion.eps}}
\scalebox{1.0}{\includegraphics{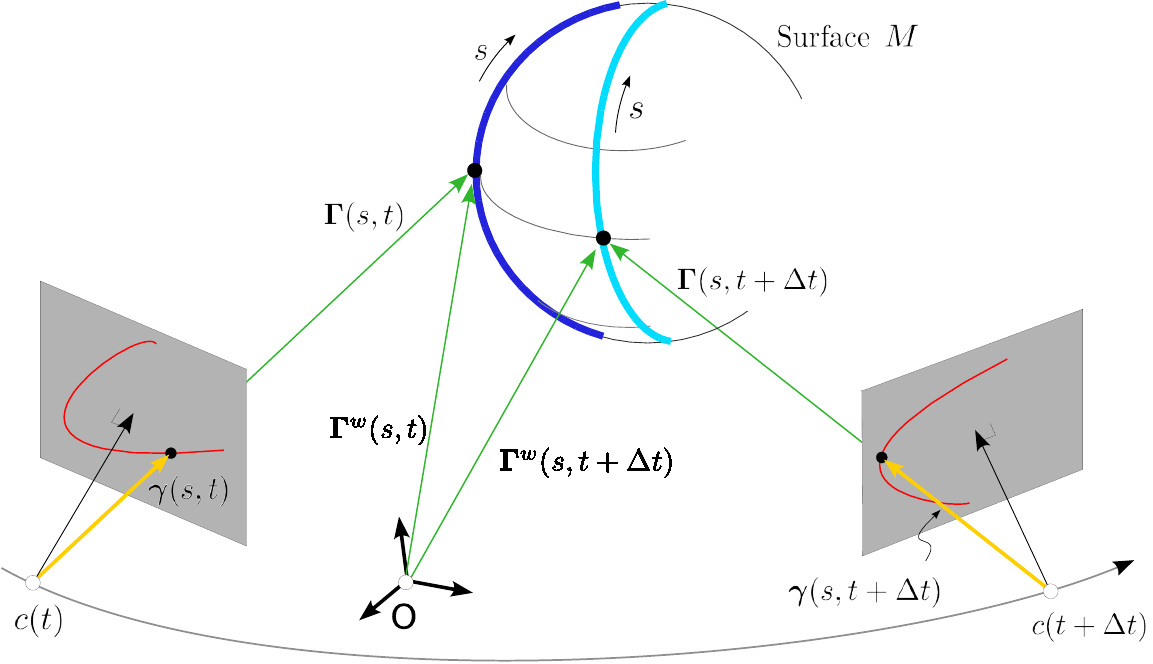}}
\caption{% 
Multiview formulation of continuous camera motion and a possibly moving
contour.
}\label{fig:mview:rig:continuous}
\end{figure}
Let the camera position over time (\emph{camera orbit}) be described by the
space curve
$\bc(t)$ and the camera orientation
by a rotation matrix $\rot(t)$.
\nomenclature[02000]{$t$}{
$(seconds)$ Time, parametrizing camera or object motion
}%
\nomenclature[03900]{$\rot(t)$}{The rotation matrix from
coordinates at time~$t$ to the world coordinates}%
For simplicity, and without loss of
generality, we take the camera coordinate system at $t=0$ to be the world
coordinate system, \ie, 
$\bc(0) = 0$, $\transl(0) = 0$, and $\rot(0) = \id$,
where $\id$ is the identity matrix. Also, a stationary point can be modeled in
this notation by making $\Gama^w(t) = \Gama^w(0) = \Gama_0$.

A differential camera motion model using
time derivatives of $\rot(t)$ and $\transl(t)$ can be used to relate frames in a
small time interval. Since
$\rot\rot^\top = \id$,
\begin{equation}
\frac{d\rot}{dt}\rot^\top + \rot \frac{d\rot}{dt}^\top = 0,
\end{equation}
which implies that $\skewm \OO \doteq \frac{d\rot}{dt}\rot^\top$ is a
skew-symmetric matrix, explicitly written as
\begin{equation}
\skewm \OO = 
\begin{bmatrix}
0 & -\Omega_z & \Omega_y \\ \Omega_z & 0 & -\Omega_x \\ -\Omega_y & \Omega_x & 0 
\end{bmatrix},
\end{equation}
so that $\frac{d\rot}{dt} = \skewm\OO \rot$.
Denote $\OO = \left[\Omega_x,\, \Omega_y,\,
\Omega_z\right]^\top$ as a vector form characterization of $\skewm\OO$.
Similarly, the second-derivative of $\rot(t)$ is represented by only three
additional numbers $\frac{d\skewm\OO}{dt}$, so that
\begin{equation}\label{eq:rot:tt:atzero}
\frac{d^2\rot}{dt^2} = \frac{d\skewm \OO}{dt} \rot  + \skewm \OO
\frac{d\rot}{dt}  = 
\frac{d\skewm \OO}{dt} \rot  + \skewm \OO^2
\rot. 
\end{equation}
Thus, a second-order Taylor approximation of the camera rotation matrix using
$R(0) = \id$ is
\begin{equation}
\rot(t) \approx \id + \skewm \OO (0) t + \frac{1}{2} \left[ \frac{d\skewm \OO}{dt}(0) +
\skewm\OO^2(0) \right]t^2.
\end{equation}
Similarly, the camera translation can be described by a differential model
\nomenclature[04001]{$\OO = \frac{d\rot}{dt}(t)\rot^\top(t)$}{
($rad/s$)
The first-order approximation of
Rotation matrix, $\OO = \frac{d\rot(t)}{dt}$,  represented by rotation velocities
$[\Omega_x,\,\Omega_y,\,\Omega_z]^\top$ about the
x,y, and z axis, respectively.}%
\nomenclature[04500]{$\skewm \OO$}{(matrix, $rad/s$)
Entries of $\OO$ arranged into a
skew-symmetric matrix such that $\skewm \OO \mathbf v = \OO\times\mathbf
v$ for any vector $\mathbf v$\nomrefpage}%
\nomenclature[06000]{$\VV(t) = \frac{d\transl}{dt}(t)$}{(vector, $m/s$)
Tangential velocity vector to the curve
$\transl(t)$, \ie,~$\frac{d\transl}{dt}$}%
\nomenclature[07001]{$\VVspeed(t) = \norm{\frac{d\transl}{dt}(t)}$}{$(m/s)$ Tangential velocity to
the curve $\transl(t)$}%
\nomenclature[08000]{$\ttransl(t)$}{
(unit vector) Unit tangent vector to the curve $\bc(t)$, in camera coordinates
at time $t$.
}%
\nomenclature[08001]{$\ttransl^w(t)$}{
(unit vector) Unit tangent vector to the curve $\bc(t)$, in world coordinates.
}%
\begin{equation}\label{eq:vv:def}
\VV(t) \doteq \frac{d\transl}{dt}(t) = -\skewm\OO(t)\rot(t)\bc(t) -
\rot(t)\frac{d\bc}{dt}(t),
\qquad\qquad
\VV(0) = -\frac{d\bc}{dt}(0),
\end{equation}
and
\begin{equation}
\frac{d\VV}{dt}(t) = \frac{d^2\transl}{dt^2}(t) = -\frac{d^2 \rot}{dt^2}(t)\bc(t) -
2\frac{d\rot}{dt}(t)\frac{d\bc}{dt}(t) - \rot(t)\frac{d^2\bc}{dt^2}(t),
\end{equation}
which at $t=0$ gives $\frac{d\VV}{dt}(0) = -2\skewm\OO(0)\frac{d\bc}{dt}(0) -
\frac{d^2\bc}{dt^2}(0)$.

The choice of whether to adopt the Taylor approximation of $\bc(t)$ or
$\transl(t)$ as primary is entirely dependent in which domain the higher
derivatives are expected to diminish, giving
\begin{align}
\transl(t) \approx \VV(0)\,t + \frac{1}{2}\VV_{t}(0)\,t^2,
\qquad
\bc(t) \approx -\VV(0)t + \frac{1}{2} \left[ -\VV_t(0) +
2\skewm\OO(0)\VV(0) \right]t^2.
\end{align}

\begin{table}
  \renewcommand{\arraystretch}{1.4}
  \renewcommand{\tabcolsep}{0.1cm}
  \begin{center}
  \scriptsize
  \begin{tabular}{|c|l||c|l|}
  \hline
  \multicolumn{1}{|c|}{\textbf{Symbol}} &
  \multicolumn{1}{c||}{\textbf{Description}}&
  \multicolumn{1}{|c|}{\textbf{Symbol}} &
  \multicolumn{1}{c|}{\textbf{Description}}\\\hline\hline
  $\Gama^w$ & 3D point in the world coordinate system &
  $\t$ & Image curve tangent $\t = \gama'/g$\\\hline
  $\Gama$ & 3D point in the camera coord. syst.  {\tiny$\Gama = \rot\Gama^w + \transl$} &
  $\n$ & Image curve normal $\n = \t^\perp$\\\hline
  $\rot$ & Rotation matrix: world to camera coordinates &
  $\kappa$ & Curvature of the image curve $g\kappa\n = \t'$\\\hline
  $\transl$ & Translation vector: world to camera coord. {\tiny $\transl =
  -R\bc$} &
  $S$, $\tilde S$ & Space curve arbitrary parameter \& arclength, resp.\\\hline
  $\bc$ & The camera center &
  $G$ & Space curve speed of parametrization $G = \|\Gama'\|$\\\hline
  $\skewm \OO$ & $\frac{d\rot}{dt} = \skewm \OO \rot$ &
  $\T$, $\T^w$ & Space curve tangent camera \& world coord., resp.\\\hline
  $\OO$ & Vector form of the 3 entries of $\skewm \OO$ &
  $\N$, $\N^w$ & Space curve normal: camera \& world coord., resp.\\\hline
  $\VV$ & $\VV = \frac{d\transl}{dt} = \skewm\OO\transl - \rot\bc_t$, also 
  $\VV = [V_x,\,V_y,\,V_z]^\top$
  &
  $\B$, $\B^w$ & Space curve binormal: camera \& world coord., resp.\\\hline
  $\depth$ & Depth of image point $\Gama = \depth\gama$ &
  $\e_1$, $\e_2$, $\e_3$ & Basis vectors of the camera coordinate system\\\hline
  $\gama$ & 2D point in normalized image coordinates &
  $\e_1^w$, $\e_2^w$, $\e_3^w$ & Basis vectors of the world coordinate system\\\hline
  $\gama_{im}$ & 2D point in pixel image coordinates &
  $'$ & Diff. with resp. to $S$ or $s$, depending on context\\\hline
  $s$, $\tilde s$ & Image curve arbitrary parameter \& arclength, resp. &
  $\dot\ $  & Diff. with resp. to arclength $\tilde S$ or $\tilde s$ \\\hline
  $g$ & Image curve speed of parametrization $g = \|\gama'\|$ & 
  $\theta$ & The angle $\measuredangle(\T,\gama)$\\\hline
  $(u,v)$ & Image velocities $\gama_t = [u,\,v,\,0]^\top$ & &\\\hline
 \end{tabular}
 \end{center}\label{tab:notation} 
 \caption{Notation.} 
\end{table}

\subsection{Relating World and Camera-Centric Derivatives.}

\begin{proposition}
The velocity of a 3D point $\Gama(t)$ in
camera coordinates, $\Gama_t(t)$, is related to its velocity in the world
coordinates $\Gama^w_t(t)$ by
\begin{empheq}[left=\empheqlbrace]{align}\label{eq:3D:point:velocity:allt}
\Gama_t &=  \skewm\OO\rot\Gama^w + \rot\Gama^w_t + \VV = 
    \skewm\OO\Gama + \rot\Gama^w_t - \rot\bc_t,\\
%\label{eq:3D:point:velocity}
\Gama_t &= \skewm\OO\rot\Gama_0 + \VV = \skewm\OO\Gama - \rot\bc_t,
\ \ \ \ \text{for a fixed point, $\Gama^w = \Gama_0$.}
\label{eq:3d:velocity:camera}
% \ \ \ \text{for any $t$.}
\end{empheq}
\end{proposition}
\begin{proof}
Differentiating Equation~\ref{eq:coord:transf:RT} with respect to time,
\begin{align}
\Gama_t &= \rot_t\Gama^w + \rot\Gama^w_t + \transl_t\\
&= \skewm\OO \rot\Gama^w + \rot\Gama^w_t + \VV\\
&= \skewm\OO(\Gama - \transl) + \rot\Gama^w_t + \VV\\
&= \skewm\OO\Gama + \rot\Gama^w_t + \VV - \skewm\OO\transl.
\end{align}
The result follows from using $\transl = -\rot \bc$,
\begin{equation}
\VV = \transl_t = -\rot_t\bc - \rot\bc_t = -\skewm\OO\rot\bc - \rot\bc_t =
\skewm\OO\transl - \rot\bc_t.
\end{equation}
\end{proof}
%
% TODO this draftnote might be useful in research:
%
%
%\draftnote{todo: do the same in~\eqref{eq:3D:point:velocity:allt} for $\Gama^w(t) \neq 0$;\\
%obs: When generalizing the formulas for any $t$, some authors~\cite{Ma:Soatto:etal:book} define $\mathcal V(t) :=  -
%\skewm\Omega(t)\transl(t)- \VV(t)$}%
%
%\noindent .
%It follows from Equation~\eqref{eq:3d:velocity:camera} that the first order
%approximation to $\Gama(t)$ at $t=0$ when the world coordinates are placed at
%the camera at $t=0$, so that $\rot(0) = \id$ and $\frac{d \rot}{dt} (0) = \skewm\Omega$, is
%\begin{equation}
%\Gama(t) \approx \Gama_0 + \left( \skewm\Omega(0)\Gama_0 - \VV(0) \right)t\,.
%\end{equation}

\subsection{Stationary and Non-Stationary Contours}\label{sec:apparent:contour:basics}
It is important to differentiate between image contours arising from a space
curve that is changing at most with a rigid transform (stationary contours),
\eg, reflectance contours and sharp ridges, 
and image curves arising from deforming space curves (non-stationary contours),
\eg, occluding contours,  the \emph{contour
generators} projecting to \emph{apparent contours}.
Stationary contours are characterized by $\Gama_t^w = 0$ while for 
occluding contours 
the viewing direction $\Gama(S,t)$ is tangent to the surface $\surface$ with
surface normal $\N$ ($\N^w = \rot^\top \N$)
\begin{align}\label{eq:occlusion:condition}
\Gama^\top\N = 0,\qquad \text{or}\qquad (\Gama^w-\bc)^\top\N^w = 0.
\end{align}
For the image curve $\gama(s,t)$ arising from 
the occluding contour, Figure~\ref{fig:mview:rig:continuous}, the
normal $\N$ to $\surface$ at an occluding contour~\cite{Giblin:Motion:Book} can
be consistently taken as $\N = \frac{\gama\times\t}{\|\gama\times\t\|}$.

Unless otherwise stated, we assume that the parametrization $\Gama^w(S,t)$ of $\surface$ is regular for
occluding contours, so that $\Gama^w_S(S,t)$ and $\Gama^w_t(S,t)$ form the
tangent plane to $\surface$ at $\Gama^w(S,t)$, and $t$ can be seen as a
spatial parameter~\cite{Giblin:Weiss:IVC1995}. The
correlation of the parametrization $S$ of $\Gamma$ at time $t$ to that of nearby
times is captured by
$\Gama_t^w(S,t)$, which 
is orthogonal to $\N^w$ (since $\N^w$ is orthogonal to the tangent plane),
but is otherwise arbitrary as a one dimensional choice. It is common to require
that
$\Gama^w_t(S,t)$ lay on the (infinitesimal) epipolar plane, spanned by
$\Gama^w(S,t)$, $\bc(t)$, and $\bc_t(t)$, referred to as the \emph{epipolar
parametrization}~\cite{Giblin:Motion:Book,Giblin:Weiss:IVC1995},
\begin{equation}\label{eq:epipolar:param:eq}
\Gama_t^w\times(\Gama^w-\bc) = 0,\qquad\text{or}\qquad \Gama_t^w =
\lambda(\Gama^w - \bc)\text{ for some $\lambda$.}
\end{equation}

\section{Projecting Differential Geometry Onto a Single
View}\label{sec:single:view:projection}

This section relates the intrinsic differential-geometric attributes of the
space curve and those of its perspective image curves.
Specifically, the derivatives $\Gama'$, $\Gama''$, and $\Gama'''$ are 
first expressed in terms of the differential geometry of $\Curve$,
namely $\{\T,
\N, \B, K,\dot{K}, \tau\}$, and second, they are
expressed in
terms of the differential geometry of $\gama$, namely $\{\t, \n, \kappa,
\dot{\kappa}\}$ using $\Gama = \depth\gama$.  Note that $\dot{K}$ and $\dot{\kappa}$ are
both intrinsic quantities. In equating these two expressions, we relate
$\{\T, \N, \B, K, \dot{K}, \tau\}$ to $\{\t, \n, \kappa, \dot{\kappa}\}$. Our purpose is to eliminate
the dependence on the parametrizations $(g,G)$, and depth $\depth$, \ie, final expressions
do not contain these unknowns nor their derivatives $(g, g',
g'')$, $(G, G', G'')$, or unknown depth and its derivatives
$(\depth, \depth', \depth'', \depth''')$. Intrinsic camera parameters are dealt
with in Section~\ref{sec:intrinsics}.

\begin{proposition}
%
\begin{comment}\hspace{-3pt}\footnote{Recall that
$\kappa' = g\dot{\kappa}$ and $K' =
G\dot{K}$, where $\dot{\kappa}(s) = \frac{dk}{d\tilde{s}}(s)$ and $\dot{K}(s)
= \frac{dK}{d\tilde{S}}(s)$.} \end{comment}
%
%
$\{\T,\, \N,\, \B,\, K,\,
\dot{K},\, \tau,\, G,\, G',\, G''\}$ are related to $\{\gama,\,\t,\, \n,\,
\kappa,\, \dot{\kappa},\, g, \, g',$ $g'',\, \depth,\, \depth',\, \depth'',\, \depth'''\}$ by
\begin{empheq}[left=\empheqlbrace]{align}
G\T &= \depth'\gama + \depth g \t  \label{eq:T:t} \\
G'\T + G^2K\N &=  \depth''\gama + (2\depth'g + \depth g')\t + \depth
  g^2\kappa\n \label{eq:N:n}\\
\begin{split}
(G'' - G^3K^2)\T &+ (3GG'K + G^3\dot{K})\N + G^3K\tau\B = \\
\depth'''\gama + [3\depth''g + 3\depth'g' & + \depth(g'' - g^3\kappa^2)] \t
+\, [3\depth'g^2\kappa + \depth(3gg'\kappa + g^3\dot{\kappa})] \n,
\end{split}
\end{empheq}
where $\Gamma$ and $\gamma$ are linked by a common parameter 
via projection $\Gama = \depth\gama$.
\end{proposition}
\begin{proof}
First, writing $\Gama'$, $\Gama''$, and $\Gama'''$ in the Frenet frame of $\Gama$ as
\begin{empheq}[left=\empheqlbrace]{align}
  \Gama' &= G\T \label{eq:Gama:prime:T}\\
  \Gama'' &= G'\T + G^2K\N\\
  \Gama''' &= (G'' - G^3K^2)\T + (3GG'K + G^2K')\N + G^3K\tau\B.
  \label{eq:Gama:3prime:B} 
\end{empheq}
Note that when expressed with respect to the arc-length of $\Gama$, i.e., 
$G \equiv 1$, simple expressions result: 
\begin{empheq}[left=\empheqlbrace]{align}
\boldsymbol{\dot{\Gama}} &= T \label{eq:Gama:dot:T}\\
\boldsymbol{\ddot{\Gama}} &= K\N\\
\boldsymbol{\dddot{\Gama}} &= -K^2\T + \dot{K}\N +
K\tau\B\label{eq:Gama:3dot:B}.
\end{empheq}
Second, differentiating $\Gama = \depth\gama$ gives
\begin{empheq}[left=\empheqlbrace]{align}
\Gama' &= \depth'\gama  + \depth\gama'\label{eq:gamaprime}\\
\Gama'' &= \depth''\gama + 2\depth'\gama' + \depth\gama'' \label{eq:Gama2prime}\\
\Gama''' &= \depth'''\gama + 3\depth''\gama' + 3\depth'\gama'' +
\depth\gama'''. \label{eq:Gama3prime}
\end{empheq}
This can be rewritten using expressions for the derivatives of $\gama$, which are
\begin{empheq}[left=\empheqlbrace]{align}
\gama' &= g\t \label{eq:gama:prime:t}\\
\gama'' &= g'\t + g^2\kappa\n \label{eq:gama:2prime:tn}\\
\gama''' &= (g'' - g^3\kappa^2)\t + (3gg'\kappa + g^2\kappa')\n. \label{eq:gama:3prime:n}
\end{empheq}
Thus, $\Gama',\, \Gama''$, and $\Gama'''$ can be written in terms of $\gama,\,
\t,\, \n,\, \kappa,\,  
\dot{\kappa},$ $g$, $g'$, $g''$, $\depth,\, \depth',\, \depth'',\, \depth'''$ as
\begin{empheq}[left=\empheqlbrace]{alignat=2}
\Gama' &= \depth'\gama & &+ \depth g \t \label{eq:Gama:prime:t}\\
\Gama'' &=  \depth''\gama & &+ (2\depth'g + \depth g')\t + \depth
g^2\kappa\n\\
\Gama''' &= \depth'''\gama & &+ [3\depth''g + 3\depth'g' +
\depth(g'' - g^3\kappa^2)]\t \notag\\
& & &+ [3\depth'g^2\kappa + \depth(3gg'\kappa + g^3\dot{\kappa})] \n.\label{eq:Gama:3prime:n}
\end{empheq}
Equating~(\ref{eq:Gama:prime:T}-\ref{eq:Gama:3prime:B}) and
(\ref{eq:Gama:prime:t}-\ref{eq:Gama:3prime:n}) proves the proposition. 
\end{proof}

\begin{corollary}Using the arc-length $\tilde{S}$ of the space curve as the common
parameter, i.e., when $G \equiv 1$, we have
\begin{empheq}[left=\empheqlbrace]{alignat=2}
\T &= \depth'\gama & &+ \depth g \t  \label{eq:T:t:stilde} \\
K\N &=  \depth''\gama & &+ (2\depth'g + \depth g')\t + \depth
  g^2\kappa\n \label{eq:N:n:stilde}\\
-K^2\T + \dot{K}\N + K\tau\B &= \depth'''\gama & &+ [3\depth''g + 3\depth'g' +
\depth(g'' - g^3\kappa^2)] \t \notag\\
& & &+ [3\depth'g^2\kappa + \depth(3gg'\kappa + g^3\dot{\kappa})] \n. \label{eq:TNB:tn}
\end{empheq}
\end{corollary}

\noindent\textbf{First-Order Differential Geometry.} We are now in a position to derive the first-order differential attributes of
the image curve  $(g,\t)$ from that of the 
space curve ($G$, $\T$). Note from~\eqref{eq:T:t} 
or~\eqref{eq:T:t:stilde}
that $\T$ lies on the plane spanned by $\t$ and $\gama$, i.e., $\T$ is
a linear combination of these vectors.  An exact relationship is expressed
bellow.

\begin{theorem}\label{th:tfromT}
Given the tangent $\T$ at $\Gama$ when $\T$ is not aligned with $\gama$, then
the corresponding tangent $\t$ and normal $\n$ at
$\gama$ are determined by
\begin{align}
\t = \frac{\mathbf{T} - T_z\gama}
{\|\mathbf{T} - T_z\gama\|} \label{eq:tfromT},\qquad
\n = \t^\perp \doteq \t \times \f.
\end{align}
\end{theorem}
\begin{proof}
Equation~\ref{eq:T:t} states that $\T$, $\t$, and $\gama$ are coplanar. Taking
the dot product with $\e_3$ and using $\e_3^\top\gama = 1$,
$\e_3^\top \t = 0$, and $\depth' =
GT_z$ (Equation \ref{eq:depth:derivs}), isolate $\t$ in the original
equation as
\begin{align}
\t = \frac{1}{\depth g}\left[ G\T - \depth'\gama \right] = \frac{G}{\depth g}
[\T - T_z\gama]
\label{eq:t:from:T:proof}
\end{align}
and the result follows by normalizing. The formula for the normal comes from the
fact that it lies in the image plane, therefore being orthogonal to both $\t$ and
$\e_3$.
\end{proof}

Observe that the depth scale factor $\depth$ is not needed to find $\t$ from $\T$. 
Moreover, when $\gama$ and $\T$ are aligned for a point on $\gamma$,
Equation~\ref{eq:t:from:T:proof} still holds, but implies that $g=0$ and $\t$
is undefined, \ie, that the image curve will have stationary points and
possibly corners or cusps. Stationary points are in principle not detectable
from the trace of $\gama$ alone, but by the assumption of general position these
do not concern us.

A crucial quantity in relating differential geometry along the space
curve to that of the projected image curve is the \emph{ratio of speeds of
parametrizations}
$\frac{g}{G}(s)$. The following theorem derives the key result that this
quantity is \emph{intrinsic} in that it does not depend on the parametrization
of $\Gamma$ or of $\gamma$, thus allowing a relationship between the differential
geometry of the space and image curves.

\begin{theorem}\label{thm:g:stilde}
The ratio of speeds of the projected 2D curve $g$ and of the
3D curve $G$ at corresponding points is an intrinsic quantity given by
\begin{equation}\label{eq:gG}
\frac{g}{G} = \frac{ \|\T - T_z\gama\|}{z}
\qquad\text{ or }\qquad g = \frac{\|G\T - \depth'\gama\|}{\depth},
\end{equation}
\ie, it does not depend on the parametrization of $\Gama$ or of $\gama$.
\end{theorem}
\begin{proof}
Follows from a dot product of Equation~\ref{eq:t:from:T:proof} with $\t$ and
dividing by $\depth = z$.
\end{proof}

\begin{comment}
\begin{corollary} The speed of an image curve in terms of the arc-length of the space
curve, and vice-versa, are respectively given by
\begin{align}\label{eq:g:stilde}
g(\tilde{S}) &= \frac{ \|\T - (\f^\top\T)\gama\|}{\f^\top\Gama}, &
G(\tilde{s}) &= \frac{\f^\top\Gama}{ \|\T - (\f^\top\T)\gama\|}.\\
\intertext{This implies that the arclengths of the image and space curves can be expressed
as:}
\tilde s(\tilde S) &= \int_{\tilde{S}_0}^{\tilde{S}} g(\tilde S)\,d\tilde
S,
&
\tilde S(\tilde s) &= \int_{\tilde{s}_0}^{\tilde{s}} G(\tilde s)\,d\tilde s
.
\label{eq:arclength}
\end{align}
\end{corollary}
\begin{proof}
Set $g(\tilde s)=1$ or $G(\tilde S)=1$ in~\eqref{eq:gG}.
\end{proof}
\end{comment}

\noindent \textbf{Second-Order Differential Geometry.} The curvature of
an image curve can be derived from the curvature of the space curve, as shown by the next theorem.
\begin{theorem}\label{thm:curvature:3d:2d}
The curvature $\kappa$ of a projected image curve
is given by
\begin{align}\label{eq:kfromK}
\kappa = \left[\frac{\N-N_z\gama}{\depth g^2} \cdot \n\right] \,K ,
\text{ when $G\equiv 1$,}
\qquad{or}\qquad
\kappa  = \left(\frac{G}{g}\right)^2\left[ \frac{\N^\top(\gama\times\t)}{\depth
} \right]\, K,
\end{align}
where $g$ and $\frac{G}{g}$ are given by Equation~\ref{eq:gG}, 
and $\depth=\f^\top\Gama$. 
%The normal orientation is
%chosen to be consistent with positive curvature,
% xxx for thesis
%~\footnote{In this convention, the
%image curvature is always positive; the normal changes orientation in going from
%convexities to concavities in order to assure this. This is consistent with the
%approach in this paper of representing image curves as planar curves
%embedded in 3D; their basic modeling is the same as space curves with
%zero torsion. In effect, the approach is equivalent to using projective coordinates
%of 2D points interpreted as 3D rays, but in an Euclidean fashion.}:
%\begin{equation}
%\n = \text{sign}(\bar \kappa)\, \bar\n.
%\end{equation}
The tangential acceleration of a projected curve with respect to the arc
length of the space curve is given by
%\draftnote{Also, we should derive $G'$ since it is used in Taylor}
\begin{equation}\label{eq:gprime}
\frac{dg}{d\tilde{S}} = \frac{[\N - N_z\gama]^\top\t\,K}
{\depth} -2g\frac{T_z}
{\depth}, 
\qquad\text{or}\qquad
\frac{dg}{d\tilde{S}}  = - \frac{K\N^\top(\gama\times\n)}{\depth} 
 - 2g \frac{T_z}{\depth}.
\end{equation}
\end{theorem}
\begin{proof}
Using Equation~\ref{eq:depth:derivs} in
Equation~\ref{eq:N:n}  leads to
\begin{align}
G'\T + G^2K\N &= 
 (G'T_z + G^2KN_z)\gama + 2GT_zg\t + \depth g'\t  +  \depth
g^2\kappa\n . \label{eq:k:K:proof}
\end{align}
First, in the case of $G \equiv 1$ curvature $\kappa$ can be isolated by taking
the dot product of the last equation with $\n$ which gives the curvature
projection formula~\eqref{eq:kfromK}. By instead taking the
dot product with $\gama\times\t$ we arrive at the alternative
formula, since the only remaining terms are those containing $\N$ or $\n$,
\begin{equation}
G^2K\N^\top(\gama\times\t) = \depth g^2 \kappa \n^\top(\gama\times\t).
\end{equation}
Isolating $\kappa$ and using $\n^\top(\gama\times\t) = \gama^\top(\t\times\n) =
\gama^\top\e_3 = 1$ gives the desired
equation.
Second, the term $g'$ can be isolated by taking the dot
product with $\t$ or with $\gama\times\n$, giving the first and second formulas,
respectively, noting that $\t^\top(\gama\times\n) = -1$.
\end{proof}
\begin{comment}
It follows that if $\T$ is in the direction of $\gama$, \emph{i.e.}, if we are
projecting in
the direction of the 3D tangent, then $\t=0$, so, from our curvature formula, we
have a cusp.
\end{comment}

Note that formulas for the projection of 3D tangent and curvatures onto 2D
tangent and geodesic curvature appear
in~\cite{Cipolla:Zisserman:1992} and~\cite[pp.~73--75]{Giblin:Motion:Book}, but an actual image
curvature was not determined there.
That the curvature of the space curve is related to the curvature of the
projected curve was derived in previous
work~\cite{Li:Zucker:2003,Robert:Faugeras:1991}, but our proof is much simpler and
more direct. Moreover, our proof methodology generalizes to relating higher
order derivatives such as curvature derivative and torsion, as shown below.

\begin{theorem}\label{thm:k:prime:from:Gama}
%Given $\Gama$, $\T$, $\N$, $\B$, $K$, $\dot{K}$, $\tau$, 
The curvature
derivative of a projected image curve $\gama$ is derived from the
local third-order differential geometry of the space curve as follows
\begin{equation}\label{eq:kprime:from:Ktau}
\dot{\kappa} = \frac{[\dot{K}\N +K\tau\B]^\top (\gama\times\t)}{\depth g^3} -
3\kappa\left(\frac{T_z}{\depth g} + 
\frac{g'}{g^2}\right),
\end{equation}
\end{theorem}
assuming $G \equiv 1$.
\begin{proof}
Taking the scalar product of Equation~\ref{eq:TNB:tn} with $\gama\times\t$, and using
$\T^\top(\gama\times\t) = 0$ and $\n^\top(\gama\times\t) = \gama^\top(\t\times\n) =
\gama^\top\e_3 = 1$,
\begin{align}
[\dot{K}\N + K\tau\B]^\top(\gama\times\t) &= 3\depth'g^2\kappa +  \depth(3gg'\kappa +
g^3\dot{\kappa}),
\label{eq:kdot:first:derivation}
\shortintertext{which using $\depth' = T_z$ gives}
3T_z g^2\kappa +  \depth(3gg'\kappa +
g^3\dot \kappa) &= [\dot K\N + K\tau\B]^\top(\gama\times\t) .
\end{align}
Isolating $\dot{\kappa}$ gives the desired result. Since both
$g$ and $g'$
are available from Equations~\ref{eq:gG} and~\ref{eq:gprime}, the theorem follows.
\end{proof}

\subsection{Intrinsic Camera Parameters and Differential
Geometry}\label{sec:intrinsics}
This section derives the relationship between the intrinsic differential
geometry $\{\t,\, \n,\, \kappa,\, \dot \kappa\}$  of the curve in normalized
image coordinates to those in image pixel coordinates,
$\{\t_{im},\,\n_{im},\,\kappa_{im},\,\dot\kappa_{im}\}$. 
Using the intrinsic parameter matrix $\KK_{im}$ relating $\gama_{im} = \KK_{im}\gama$, 
by the linear Equation~\ref{eq:intrinsic:parameter:transf}. 

\begin{theorem}
The intrinsic quantities $\{\t,\,\n,\,\kappa,\,\dot\kappa\}$ and
$\{\t_{im},\,\n_{im},\,\kappa_{im},\,\dot\kappa_{im}\}$ under linear
transformation $\gama_{im} = \Kim\gama$ are related by
\begin{empheq}[left=\empheqlbrace]{align}
g_{im} &= \|\Kim\t\|,\qquad \t_{im} = \frac{\Kim\t}{\|\Kim\t\|},
\qquad \n_{im} = \t_{im}\times\e_3,\\
g'_{im} &= \frac{\kappa \t^\top\Kim^\top\Kim\n}{g_{im}}, \qquad \kappa_{im} =
\frac{\n_{im}^\top \Kim\kappa\n}{g_{im}^2},\\
\dot\kappa_{im} &= \frac{1}{g^3}\n_{im}^\top \Kim (-\kappa^2\t + \dot\kappa\n) -
\frac{3g_{im}'\kappa_{im}}{g_{im}^2}.
\end{empheq}
where the speed $g_{im}$ is relative to unit speed at $\gama$.
\end{theorem}
\begin{proof}
Differentiating~\eqref{eq:intrinsic:parameter:transf} with respect to the
arc-length $\tilde s$ of
$\gama$ and using~\eqref{eq:gama:prime:t}, $\gama_{im}' = \Kim\dot\gama$ gives
\begin{align}
g_{im}\t_{im} = \Kim\t.
\end{align}
Differentiating~\eqref{eq:intrinsic:parameter:transf} a second time with respect
to $\tilde s$, and using Equation~\ref{eq:gama:2prime:tn},
\begin{equation}
g_{im}'\t_{im} + g_{im}^2\kappa_{im}\n_{im} = \Kim\kappa\n.
\end{equation}
Taking the dot product with $\t_{im}$ gives the formula for $g_{im}'$, and taking the
dot product with $\n_{im}$ gives the formula for $\kappa_{im}$.
Differentiating~\eqref{eq:intrinsic:parameter:transf} a third time with respect
to $\tilde s$, and 
using Equation~\ref{eq:gama:3prime:n}, we have
\begin{equation}
(g_{im}'' - g_{im}^3\kappa_{im}^2)\t_{im} + (3g_{im} g_{im}'\kappa_{im} + g^3_{im}\dot\kappa_{im})\n_{im} = 
\Kim(-\kappa^2\t + \dot\kappa\n).
\end{equation}
Taking the dot product with $\n_{im}$,
\begin{equation}
3g_{im} g_{im}'\kappa_{im} + g_{im}^3\dot\kappa_{im} = \n_{im}^\top
\Kim(-\kappa^2\t +
\dot\kappa\n),
\end{equation}
and isolating $\kappa_{im}$, the last result follows.
\end{proof}

The above theorem can also be used in its inverse form from $\gama_{im}$ to
$\gama$ by substituting $\Kim$ for $\Kim^{-1}$, and trivially exchanging the
sub-indices. Moreover, the theorem is generally valid for relating differential
geometry under any linear transformation in place of $\Kim$.

\input{curves-reconstruction}

\section{Projecting Differential Geometry Under Differential
Motion}\label{sec:projection:differential:motion}

The goal of this section is to relate \emph{differential observations} in a
series of images from a continuous video sequence to the \emph{differential
geometry} of the space curve. As this relationship is governed by the
\emph{differential motion} of the camera and its intrinsic parameters, we also
aim to recover scene geometry and camera motion/pose from these observations and
equations. An account of intrinsic camera calibration in this setting is
left for future work. We explore how differential scene properties are projected
onto differential image properties for points and curves, and expect future work
to apply this to surfaces. 

Differential models of camera motion observing a \emph{rigid scene} were studied
in~\cite{LonguetHiggins:Prazdny:PRSL:1980,Waxman:Ullman:IJRobotics85,
Maybank:book:1992,Ma:Soatto:etal:book,Hengel:PhdThesis:2000,
Triggs:ICCV99,Astrom:Heyden:CVPR96,Heyden:ICPR06,Baumela:etal:ICPR00,Dornaika:Sappa:LNCS06,Kahl:Heyden:ICCV01,Vieville:Faugeras:CVIU96,Hengel:etal:LNCS07,Brodsky:Fermuller:Aloimonos:IJCV00,Brodsky:Fermuller:IJCV02}.
These papers studied how the first and second-order motion of the image of fixed
points relate to a differential camera motion model. They also envisioned
recovering local 3D surface properties from the local behavior of the
velocities of projected surface points in an image neighborhood.
Differential models for \emph{nonrigid curves} observed in a monocular video sequence
were studied
in~\cite{Faugeras:Papadopoulo:IJCV93,Papadopoulo:Faugeras:ECCV96,Papadopoulo:PhD:96,Faugeras:ECCV1990},
where it was established that multiple simultaneous video sequences would be
needed. This led to a practical work in the reconstruction of nonrigid curves from
multiview video~\cite{Carceroni:Kutulakos:CVPR99,Carceroni:PhDThesis01},
exploiting temporal consistency within each video frame, as well as consistency
across the two video sequences.
Differential models of occluding contours were studied mainly
in~\cite{Cipolla:PHD:1991,Cipolla:Blake:IJCV92,Giblin:Motion:Book}, relating the deformation of
apparent contours under differential camera motion to scene properties such as
occluding contours, and a differential-geometric model of the underlying 3D
surface.

\input{flow-fixed}
\input{flow-deforming-curves}

\section{Mathematical Experiment}
To illustrate and test the proposed theoretical framework, we have devised an
experiment around a synthetic dataset constructed for this research. This
dataset has already been used for validating a pose estimation system~\cite{Fabbri:Kimia:Giblin:ECCV12}. The dataset is
composed of the following components:
\begin{enumerate}
  \item A variety of synthetically generated 3D curves (helices, parabolas,
    ellipses, straight lines, and saddle curves) with well-known parametric
    equations, as shown in Figure~\ref{fig:synth:data:sample}. 
  \item Ground-truth camera models for a video sequence around the curves.
  \item Differential geometry of the space curves analytically computed up to
    third-order (torsion and curvature derivative), using Maple when necessary.
    The dataset together with C++ code implementing these expressions from Maple
    are listed in the supplementary material Online Resource~1.
  \item The 3D curves are densely sampled, each 3D sample having attributed differential geometry
    from the analytic computation (up to torsion and curvature).
  \item A video sequence containing a family of 2D curve samples with attributed
    differential geometry is rendered by projecting the 3D samples onto a
    $500\times400$ view using the ground truth cameras. These subpixel edgels
    with attributed differential geometry simulate ideal aspects of what in
    practice could be the output of high-quality subpixel edge detection and
    grouping~\cite{Tamrakar:Kimia:ICCV07,Tamrakar:PHD:2008,Yuliang:etal:CVPR14}.
  \item Correspondence between all samples obtained by keeping track of the underlying 3D points.
  \item Specific analytic expressions for 2D differential geometry were derived using Maple 
    since these are often too long due to perspective projection. These
    expressions are also provided in the C++ code that synthesizes the dataset.
  \item C++ code implementing the formulas in this paper is also provided with
    the dataset, and can be readily used in other projects.
\end{enumerate}

For the present theoretical paper, the experiments consist of checking the
proposed expressions against the analytic expressions that are obtained by
differentiating each specific parametric equation. After projecting differential
geometry using our formulas applied to the 3D samples attributed with
differential geometry, we obtain corresponding 2D projected differential
geometry at each sample. We compare this to the differential geometry on the
curve projections analytically computed from the parametric equations, observing
a match. We then reconstruct these correspondences up to third-order
differential geometry using the proposed expressions, and observe that they
indeed match to the original analytic expressions from Maple. We have also
performed a similar experiment for the expressions involving occluding contours,
using a 3D ellipsoid and sphere.

We have observed a complete agreement between our code and the specific analytic
expressions, confirming that the formulas as presented in this manuscript are
correct. The source code of this illustrative experiment also
serves as an example of how to use the proposed framework in programming
practice, how to check for degenerate conditions stated in the theorems, among others.

\begin{figure}[ht]
   \centering
   \includegraphics[height=2in]{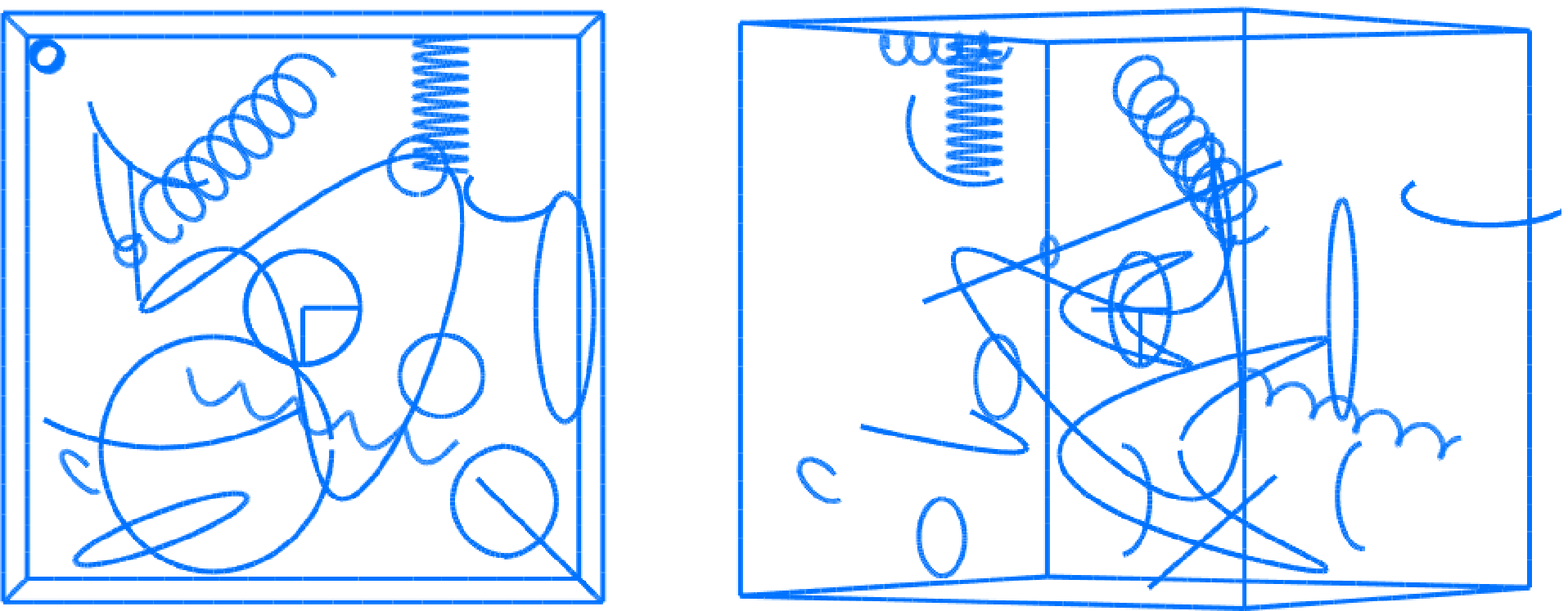}
   \caption{
   Two views of the synthetic multiview curve differential geometry
   dataset~\cite{Fabbri:Kimia:Giblin:ECCV12}.
 }
\label{fig:synth:data:sample}
\end{figure}

\section{Conclusion}
We presented a unified differential-geometric theory of projection and
reconstruction of general curves from multiple views. By gathering previously
scattered results using a coherent notation and proof methodology that scale to
expressing more sophisticated ideas, we were able to
prove novel results and to provide a comprehensive study on how the differential 
geometry of curves behaves under perspective projection, including
the effects of intrinsic parameters. For instance, we derived how the tangent,
curvature, and curvature derivative of a space curve projects onto an image, and
how the motion of the camera and of the curve relates to the projections. This
lead to the \emph{novel} result that torsion -- which characterizes 
the tri-dimensionality of space curves -- projects to curvature derivative in an
image, and the \emph{novel} result of how the parametrization of corresponding
image curves are linked across different views, up to third order to reflect
the underlying torsion.  We
also proved formulas for reconstructing differential geometry, given differential
geometry at corresponding points measured in at least two views.  In particular, this
gives the \emph{novel} result of reconstructing space curve torsion, given corresponding
points, tangents, curvatures, and curvature derivatives measured in two views.
We determined that
there are no correspondence constraints in two views -- any pair of
points with attributed tangents, curvatures, and curvature
derivatives are possible matches, as long as the basic point epipolar constraint
is satisfied. There is, however, a constraint in three or more views: from two
views one
can transfer differential geometry onto other views and enforce measurements to
match the reprojections using local curve shape, avoiding clutter. This has been
demonstrated in a recent work in curve-based multiview stereo by the
authors~\cite{Fabbri:Kimia:CVPR10,Fabbri:Kimia:CVPR16}, namely for matching linked curve fragments
from a subpixel edge detector across many views. Experiments clearly show that
differential-geometric curve structure is essential for the matching
to be immune to edge clutter and linking instability, by enforcing reprojections
to match to image data in \emph{local} shape.

This paper is part of a greater effort of augmenting multiple view geometry to
model general curved structures~\cite{Fabbri:PhD:2010}. Work on camera pose
estimation based on curves using the formulas in this paper has been recently
published~\cite{Fabbri:Kimia:Giblin:ECCV12}, and trifocal relative pose
estimation using local curve geometry to bootstrap the system is currently under
investigation.  In a complete system, once a core set of three views are
estimated and an initial set of curves are reconstructed, more views can be
iteratively registered~\cite{Fabbri:Kimia:Giblin:ECCV12}. These applications of
the theory presented in this paper and their ongoing extensions would enable a practical
structure from motion pipeline based on curve fragments, complementing interest
point technology.  We have also been working on the multiview differential
geometry of surfaces and their shading.

%\section*{Acknowledgments}
%The authors would like to thank CNPq and NSF.

{\small

\input{paper-bibs.bbl}
%\bibliographystyle{spbasic} %\bibliographystyle{IEEEtran}
%\bibliography{rf-multiview,shape-from-shading,multiview,motion,Kimia,bib-header,video,math-books,math,psych-books,metric,edge,leymarie_pami_scaffold,vision-books,vision,recognition,optical-flow,indexing,proceedings,1984}
}

%\begin{IEEEbiography}{Ricardo Fabbri}
%\end{IEEEbiography}

%\begin{IEEEbiography}{Benjamin Kimia}
%\end{IEEEbiography}

\end{document}

%% file: front-matter.tex
\title{%
Multiview Differential Geometry of Curves
}
\author{%
Ricardo~Fabbri\and 
Benjamin~B.~Kimia
}
%\authorrunning{Fabbri and Kimia}
\usdate
\institute{
Brown University, 
Division of Engineering,
Providence RI 02912, USA. \ \ \ \textbf{Draft date:} \today\ \currenttime\\
\email{\{rfabbri,kimia\}@lems.brown.edu}
}

\date{ }

\maketitle

\begin{abstract}
The field of multiple view geometry has seen tremendous progress in
reconstruction and calibration due to methods for extracting reliable point
features and key developments in projective geometry. Point features, however,
are not available in certain applications and result in unstructured point cloud
reconstructions. General image curves provide a complementary feature when
keypoints are scarce, and result in 3D curve geometry, but face challenges 
not addressed by the usual projective geometry of points and algebraic curves. We
address these challenges by laying the theoretical foundations of a framework
based on the differential geometry of general curves, including stationary
curves, occluding contours, and non-rigid curves, aiming at stereo
correspondence, camera estimation (including calibration, pose, and multiview
epipolar geometry), and 3D reconstruction given measured image curves.  By
gathering previous results into a cohesive theory, novel results were made
possible, yielding three contributions. First we derive the differential
geometry of an image curve (tangent, curvature, curvature derivative) from that
of the underlying space curve (tangent, curvature, curvature derivative,
torsion). Second, we derive the differential geometry of a space curve from that
of two corresponding image curves. Third, the differential motion of an image
curve is derived from camera motion and the differential geometry and motion of
the space curve. The availability of such a theory enables novel curve-based
multiview reconstruction and camera estimation systems to augment existing
point-based approaches. This theory has been used to reconstruct a ``3D curve
sketch'', to determine camera pose from local curve geometry, and tracking;
other developments are underway.
\end{abstract}

%% file: curves-reconstruction.tex
\section{Reconstructing Differential Geometry from Multiple
Views}\label{sec:reconstr}

In the previous section, we derived the differential geometry of a projected
curve from a space curve.
In this section, we derive the differential geometry of a
space curve $\Gamma$ from that of its projected image curves in multiple views,
namely
$\gama_i$ for camera $i$, $i = 1,\dots,N$.
In order to simplify the equations, in this section \emph{all vectors are written in
the common world coordinate basis}, including $\gama_i$.
Denote $\Gama_i := \Gama^w -
\bc_i$, namely $\Gama_i$ represents the vector from the $i^{th}$ camera
center to the 3D point $\Gama^w$ in the world coordinate system.

The reconstruction of a point on the space curve $\Curve$ from two corresponding
image curve points $\gama_1$ and $\gama_2$ can be
obtained by equating the two expressions for $\Gama^w$
given by Equation~\ref{eq:projection},
\begin{equation}
\left\{ 
\begin{array}{c}
\Gama^w - \mathbf{c}_1 = \depth_1\gama_1 \\
\Gama^w - \mathbf{c}_2 = \depth_2\gama_2
\end{array}\right.  
\qquad\implies\qquad
\depth_1\gama_1 - \depth_2\gama_2 = \mathbf{c}_2 - \mathbf{c}_1.
\label{eq:depth:reconstr}
\end{equation}
Taking the dot product with $\gama_1$, $\gama_2$, and $\gama_1\times\gama_2$
gives
\begin{equation}\label{eq:jun4:star2}
\left\{\begin{aligned}
\depth_1\gama_1\cdot\gama_1 - \depth_2\gama_1\cdot\gama_2 &= (\bc_2 -
\bc_1)\cdot\gama_1\\
\depth_1\gama_1\cdot\gama_2 - \depth_2\gama_2\cdot\gama_2 &= (\bc_2 -
\bc_1)\cdot\gama_2\\
0 &= (\bc_2 - \bc_1)\cdot(\gama_1\times\gama_2),
\end{aligned}\right.
\end{equation}
which gives
\begin{equation}\label{eq:jun4:star3}
\left\{\begin{aligned}
\depth_1 &= \frac{(\bc_2 - \bc_1)\cdot\gama_1\,(\gama_2\cdot\gama_2) - (\bc_2 -
\bc_1)\cdot\gama_2\,(\gama_1\cdot\gama_2)}{
(\gama_1\cdot\gama_1)(\gama_2\cdot\gama_2) - (\gama_1\cdot\gama_2)^2}\\
\depth_2 &= \frac{(\bc_2 - \bc_1)\cdot\gama_1\,(\gama_1\cdot\gama_2) - (\bc_2 -
\bc_1)\cdot\gama_2\,(\gama_1\cdot\gama_1)}{
(\gama_1\cdot\gama_1)(\gama_2\cdot\gama_2) - (\gama_1\cdot\gama_2)^2},
\end{aligned}\right.
\end{equation}
provided that $(\bc_2 - \bc_1)\cdot(\gama_1\times\gama_2) = 0$. This is
precisely the
well-known fact that this system of three equations in two unknowns $\depth_1$
and $\depth_2$ can only be solved if the lines $\bc_1\gama_1$ and $\bc_2\gama_2$
intersect. 
\begin{comment}A solution exists if and only if the lines $c_1p_1$ and $c_2p_2$
intersect. 
%This implies the classic epipolar constraint, strongly restraining the space of 
%corresponding points into epipolar lines. 
A least squares reconstruction can be found when the two lines only nearly
intersect due to imaging quantization and model error.
\end{comment}

The crucial factor in relating the differential geometry of image curves in
distinct views is the relationship between their parametrization in each view, given
in the next theorem.

\begin{proposition}
The ratio of parametrization speeds in two views of a space curve at
corresponding points is given by
\begin{equation}\label{eq:ratio:speeds:two:views}
\frac{g_1}{g_2} = \frac{\depth_2}{\depth_1} 
\frac{\|\T - (\e_{3,1}^\top\T)\gama_1\|}{\|\T - (\e_{3,2}^\top\T)\gama_2\|}
.
\end{equation}
\end{proposition}
\begin{proof}
Follows by dividing expressions for $\frac{g_1}{G}$ and $\frac{g_2}{G}$
from Equation~\ref{eq:gG}.
\end{proof}

Next, note from Equation~\ref{eq:T:t} that the unit vector
$\T$ can be written as
\begin{align}\label{eq:T:t:gama}
\T = \frac{\depth'}{G}\gama + \depth \frac{g}{G}\t.
\end{align}
Since $\T$ is a unit vector, it can be written as
\begin{align}
\label{eq:ben:T:sincos}
\T &= \cos\theta\, \frac{\gama}{\|\gama\|} + \sin\theta\,\t, 
\qquad \text{where }\ \ \cos\theta = \frac{\depth'}{G}\|\gama\|, \ \ \sin\theta =
\depth\frac{g}{G}.
\end{align}
Note that $\depth > 0$ implies that
$\sin\theta \geq 0$ or $\theta \in [0,\pi)$.
Thus the reconstruction of $\T$ from $\t$ requires the discovery of the
additional parameter $\theta$ which can be provided from tangents at two
corresponding points, as stated in the next result.

\begin{theorem}\label{thm:ben:Tfromt}
Two tangent vectors at a corresponding pair of points, namely $\t_1$ at $\gama_1$ and $\t_2$ at
$\gama_2$, reconstruct the corresponding space tangent $\T$ at
$\Gama$ as
\begin{align}\label{eq:tangent:rec:ben}
\T = \cos\theta_1 \frac{\gama_1}{\|\gama_1\|} + \sin\theta_1\t_1 =
\cos\theta_2\frac{\gama_2}{\|\gama_2\|} + \sin\theta_2\t_2,
\end{align}
and
\begin{equation}
\begin{aligned}\label{eq:ratio:speeds:two:views:ben}
\depth_1 g_1 &= \sin\theta_1 G, \qquad \depth_1'\|\gama_1\| = \cos\theta_1 G,
\qquad \frac{\depth_1'}{\depth_1g_1} =
-\frac{\t_1\cdot(\gama_2\times\t_2)}{\gama_1\cdot(\gama_2\times\t_2)} \\
\depth_2 g_2 &= \sin\theta_2 G, \qquad \depth_2'\|\gama_2\| = \cos\theta_2
G,\qquad
\frac{\depth_2'}{\depth_2g_2} =
-\frac{\t_2\cdot(\gama_1\times\t_1)}{\gama_2\cdot(\gama_1\times\t_1)},
\end{aligned}
\end{equation}
where
\begin{equation}
\begin{aligned}\label{eq:tan:tangent:ben}
\tan\theta_1 &= -\frac{1}{\|\gama_1\|}
\frac{\gama_1\cdot(\gama_2\times\t_2)}{\t_1\cdot(\gama_2\times\t_2)}, \qquad \theta_1 \in [0,\pi)\\
\tan\theta_2 &= -\frac{1}{\|\gama_2\|}
\frac{\gama_2\cdot(\gama_1\times\t_1)}{\t_2\cdot(\gama_1\times\t_1)},\qquad \theta_2 \in [0,\pi).
\end{aligned}
\end{equation}
\end{theorem}
\begin{proof}
Equating the two expressions for $\T$ from Equation~\ref{eq:ben:T:sincos}, one for each view,
gives Equation~\ref{eq:tangent:rec:ben}. Solving for
$\theta_1$ by taking the dot product with $\gama_2\times\t_2$ gives
\begin{equation}\label{eq:tangent:rec:ben:dot:gama:t}
\cos\theta_1\,\frac{\gama_1}{\|\gama_1\|}\cdot(\gama_2\times\t_2) +
\sin\theta_1\,\t_1\cdot(\gama_2\times\t_2) = 0,
\end{equation}
which leads to Equation~\ref{eq:tan:tangent:ben} and similarly for $\theta_2$.
Equation~\ref{eq:ratio:speeds:two:views:ben} follows
from equating~\eqref{eq:tangent:rec:ben:dot:gama:t} and~\eqref{eq:T:t:gama},
then taking dot products with $\gama_2\times\t_2$.
\end{proof}

\noindent\textbf{Remark:} Since $\T$ is orthogonal to both $\gama_1\times\t_1$,
and $\gama_2\times\t_2$ we have
\begin{align}\label{eq:ben:Tfromt}
\varepsilon\T &= \frac{(\t_1 \times \gama_1) \times (\t_2 \times \gama_2)} 
{\|(\t_1 \times \gama_1) \times (\t_2 \times \gama_2)\|} & \varepsilon = \pm 1,
\end{align}
where $\varepsilon$ is determined from
\begin{align}\label{eq:t:sign}
\left\{
\begin{aligned}
\varepsilon\left[\T - (\T\cdot\e_{3,1})\gama_1\right]\cdot\t_1 & > 0\\
\varepsilon\left[\T - (\T\cdot\e_{3,2})\gama_2\right]\cdot\t_2 & > 0.
\end{aligned}
\right.
\end{align}

\begin{figure*}
\centering
  \scalebox{0.55}{\includegraphics{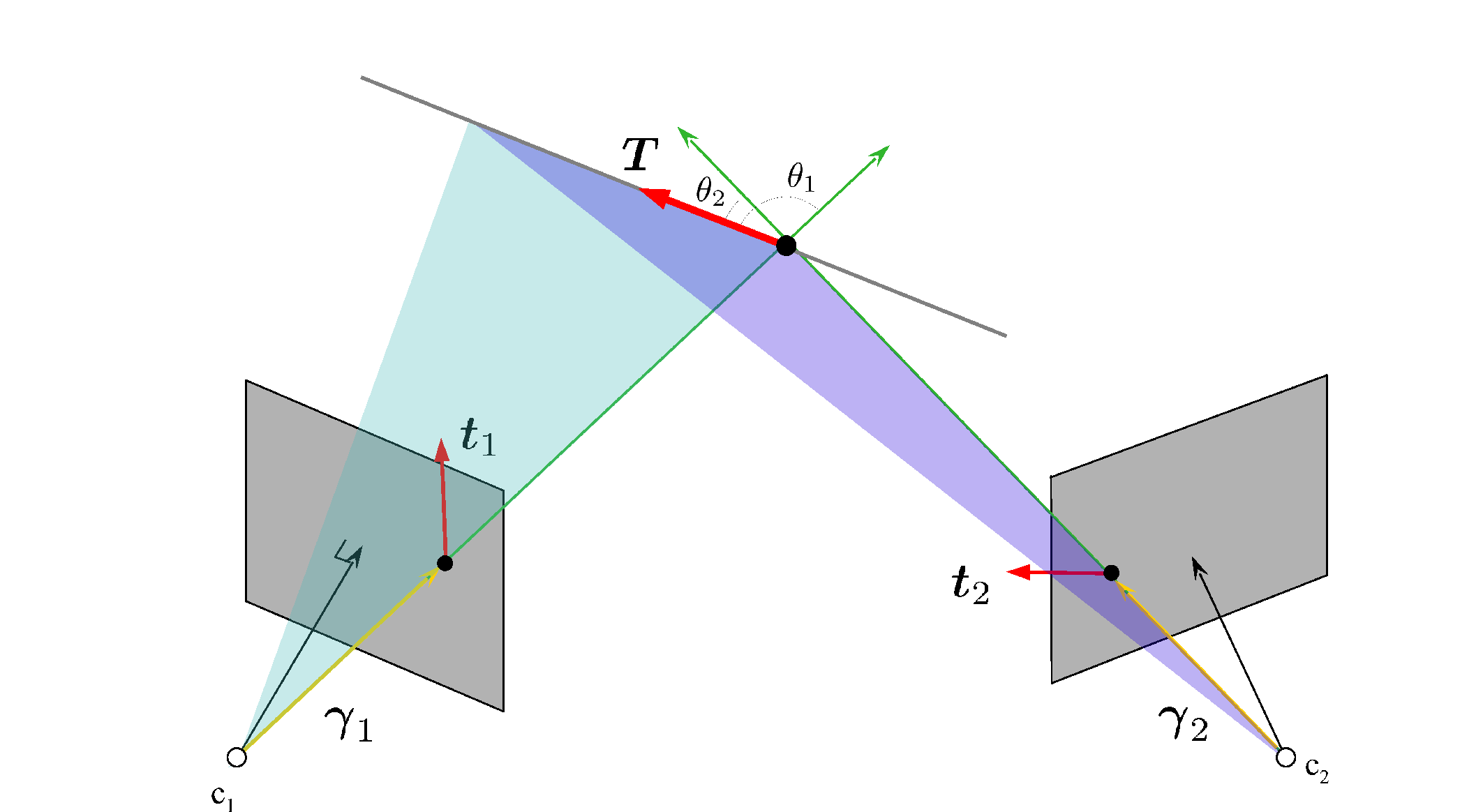}}
%  \scalebox{0.5}{\includegraphics{tangent-reconstr-old.epsi}}
\caption{% 
3D Tangent reconstruction from two views as the intersection of two planes.
}\label{fig:t:reconstr}
\end{figure*}
\noindent\textbf{Remark:} This theorem implies that \emph{any} two tangents at
corresponding points can be consistent with at least one space tangent.
Furthermore, the discovery of $\T$ does not require the \emph{a priori} solution
of $\depth_1$ or $\depth_2$.

\noindent\textbf{Remark:} An analogous tangent reconstruction expression under continuous motion
may be derived, see~\cite{Faugeras:Papadopoulo:IJCV93}.

\begin{theorem}\label{th:k:rec}
The normal vector $\N$ and curvature $K$ of a point on a space curve $\Gamma$ with 
point-tangent-curvature at projections in two views $(\gama_1,\t_1,\kappa_1)$ and
$(\gama_2,\t_2,\kappa_2)$ are given by solving the system in the vector $\N K$
\begin{equation}
\left\{ 
\begin{aligned}
G^2(\gama_1 \times \t_1)^\top\N K &= 
\depth_1g_1^2\,\,\kappa_1\\
%%%
G^2(\gama_2 \times \t_2)^\top \N K &= 
\depth_2g_2^2\,\,\kappa_2\\
%%%
\T^\top \N K &= 0,
\end{aligned}
\right.
\end{equation}
where $\T$ is given by Equation~\ref{eq:T:t:gama}, $\depth_1$ and $\depth_2$ by
Equation~\ref{eq:jun4:star3}, and $g_1$ and $g_2$ by
Equation~\ref{eq:gG}.
%The binormal vector is determined as $\B = \T\times\N$.
\end{theorem}
\begin{proof}
Taking the dot product of~\eqref{eq:k:K:proof} with $\gama\times\t$, for each
view, we arrive at the first two equations. The
third equation imposes the solution $\N K$ to be normal to $\T$.

\end{proof}

An analogous curvature reconstruction expression under continuous motion
can be derived, see~\cite{Cipolla:PHD:1991}. Theorem~\ref{th:k:rec} is a variant form of a known
result~\cite{Robert:Faugeras:1991,Li:Zucker:2003,Li:Zucker:IJCV2006}, using the
proposed unified formulation which enables a more condensed and generalizable
proof for the practical case of planar images. The next theorem leverages
this effective theoretical framework to achieve the reconstructuction of the torsion and
curvature derivative of a space curve from two image curves, which is the
central novel result of the present work.

\begin{theorem}\label{thm:Kprime:tau:from:kprime}
The torsion and curvature derivative at a point of a space curve
can be obtained  from up to third order differential geometry $\{\kappa,
\dot\kappa\}$ at a pair of corresponding points in two views
by solving for an unknown vector $\ttau$ in the system
\begin{equation}\label{eq:torsion:rec:system}
\left\{ 
\begin{aligned}
(\gama_1\times\t_1)^\top\ttau &= 3g_1^2\kappa_1\e_{3,1}^\top\T +  \depth_1(3g_1g_1'\kappa_1 +
g_1^3\dot{\kappa}_1)\\
(\gama_2\times\t_2)^\top \ttau&= 3g_2^2\kappa_2\e_{3,2}^\top\T +  \depth_2(3g_2g_2'\kappa_2 +
g_2^3\dot{\kappa}_2)\\
\T^\top \ttau &= 0 ,
\end{aligned}
\right.
\end{equation}
and by solving
for the torsion $\tau$ and curvature derivative $\dot{K}$ from $\ttau = \dot{K}\N +
K\tau\B$, \ie, 
\begin{empheq}[left=\empheqlbrace]{align}
\tau &= \frac{\ttau^\top\B}{K} \label{eq:Kprime:from:kprime}\\
\dot{K} &= \ttau^\top\N  \label{eq:tau:from:kprime},
\end{empheq}
with $\T$, $\N$, $\B$, $K$, $g_1$, $g_2$, 
$g_1'$, $g_2'$, $\depth_1$, and $\depth_2$ determined from previous derivations,
and assuming $G \equiv 1$.
\end{theorem}
\begin{proof}
Apply Equation~\eqref{eq:kdot:first:derivation} for two views, and let
$\ttau := \dot{K}\N + K\tau\B$ to get the first two equations
of~\eqref{eq:torsion:rec:system}. The last equation
of~\eqref{eq:torsion:rec:system} constrains $\ttau$ to be
orthogonal to $\T$.
\end{proof}

%% file: flow-fixed.tex
\subsection{Differential Relations for a Point}

\begin{theorem}(Moving 3D point)\label{th:derivatives:allcontours}
Let $\Gama^w(t)$ be a \emph{moving} point in space, projected onto a moving camera as
$\gama(t)$ with depth $\depth(t)$. Let the differential velocity and rotation
of the camera be $\VV$ and $\OO$, respectively, and let $\VV_t$ and
$\OO_t$ represent their derivative with respect to time $t$, respectively.
Then, the depth gradient and second derivative at $t=0$ are
\begin{empheq}[left=\empheqlbrace,right=\qquad\text{at $t=0$}]{align}
\label{eq:depth:t:Gamaw}
\depth_t &= \depth\e_3^\top(\skewm\OO\gama  + \frac{1}{\depth}\Gama^w_t + 
\frac{1}{\depth}\VV)\\
\label{eq:depthtt:moving}
\depth_{tt} &= \depth\,\e_3^\top(\skewm\OO^2 + \skewm{[\OO_t]})\gama + 2\e_3^\top\skewm\OO\Gama^w_t + 
\e_3^\top\Gama^w_{tt} +
\e_3^\top\VV_t,
\end{empheq}
and the velocity and acceleration of the projected point at $t=0$ are given by
\begin{empheq}[left=\empheqlbrace,right=\qquad\text{ at $t=0$}]{align}
\label{eq:gama:time:deriv:Gamaw:moving}
\gama_t &= \left(\skewm\OO\gama - (\e_3^\top\skewm\OO\gama)\gama\right) 
+ \frac{1}{\depth}(\Gama^w_t -
\e_3^\top\Gama^w_t\gama) +
\frac{1}{\depth}(\VV - V_z\gama)\\
\label{eq:gama:time:secondderiv:Gamaw:moving}
\gama_{tt} &= (\skewm\OO^2 + \skewm{[\OO_t]})\gama +
\frac{2}{\depth}\skewm\OO\Gama^w_t + \frac{1}{\depth}\Gama^w_{tt} +
\frac{1}{\depth}\VV_t - \frac{2\depth_t}{\depth}\gama_t -
\frac{\depth_{tt}}{\depth}\gama,
\end{empheq}
which can be simplified as
\begin{align}
\gama_{tt} &= (\skewm\OO^2 + \skewm{[\OO_t]})\gama +
\frac{2}{\depth}\skewm\OO\Gama^w_t + \frac{1}{\depth}\Gama^w_{tt} +
\frac{1}{\depth}\VV_t \notag\\
&- 2\e_3^\top\left(\skewm\OO\gama +
\frac{\VV}{\depth} + \frac{\Gama^w_t}{\depth}\right)
% \gama_t
\left( 
\frac{\VV}{\depth} - \frac{V_z}{\depth}\gama + \OO_\times\gama -
(\e_3^\top\OO_\times\gama)\,\gama
+ \frac{1}{\depth}\Gama^w_t - \frac{\e_3^\top\Gama^w_t}{\depth}\gama
\right)\notag\\
&- 
\e_3^\top\left( (\skewm\OO^2 + \skewm{[\OO_t]})\gama +
\frac{\VV_t}{\depth} + 2\skewm\OO\frac{\Gama^w_t}{\depth} + 
\frac{\Gama^w_{tt}}{\depth} \right)
\gama.
\end{align}
\end{theorem}
\begin{proof}
The image velocity $\gama_t$ is dependent on the velocity $\Gama_t$
of the 3D structure in camera coordinates, which arises from both the motion of
$\Gama^w$ and from the moving camera. Differentiating
$\Gama = \rot\Gama^w + \transl$, we get
\begin{equation}
\Gama_t = \rot_t\Gama^w + \rot\Gama^w_t + \transl_t
= \skewm\OO\rot\Gama^w + \rot\Gama^w_t +
\VV.\label{eq:Gama:Gamaw:time:deriv}
\end{equation}
Differentiating $\Gama = \depth\gama$ we get
\begin{equation}
 \Gama_t = \depth \gama_t + \depth_t \gama.\label{eq:Gama:gama:time:deriv}
\end{equation}
Equating these two expressions leads to
\begin{equation}
\depth \gama_t + \depth_t \gama = \skewm\OO\rot\Gama^w + \rot\Gama^w_t +
\VV\ \ \ \ \ \text{ for arbitrary
$t$.}\label{eq:gama:time:deriv:Gamaw:3d:2d:separated:anyt}
\end{equation}
At $t=0$ we have $\Gama^w = \Gama = \depth\gama$, leading to
\begin{equation}\label{eq:gama:time:deriv:Gamaw:3d:2d:separated}
\depth \gama_t + \depth_t \gama = \depth\skewm\OO\gama + \Gama^w_t +
\VV
\ \ \ \ \ \text{
for $t = 0$.}
\end{equation}

The depth gradient $\depth_t$ is then isolated by taking the dot product of both sides
of Equation~\ref{eq:gama:time:deriv:Gamaw:3d:2d:separated} with $\e_3$, observing that
$\e_3^\top\gama = 1$ and $\e_3^\top\gama_t = 0$, resulting in
Equation~\ref{eq:depth:t:Gamaw}. The expression for $\depth_t$ is then
substituted into Equation~\ref{eq:gama:time:deriv:Gamaw:3d:2d:separated} from
which $\gama_t$ can be isolated in the form of Equation~\ref{eq:gama:time:deriv:Gamaw:moving}.

The second order expressions $\gama_{tt}$ and $\depth_{tt}$ require
another time derivative
of Equation~\ref{eq:gama:time:deriv:Gamaw:3d:2d:separated:anyt},
\begin{equation}
\depth\gama_{tt} + 2\depth_t\gama_t +\depth_{tt}\gama = 
(\skewm\OO^2\rot + \skewm{[\OO_t]}\rot)\Gama^w
+ \skewm\OO\rot\Gama^w_t + \rot_t\Gama_t^w + \rot\Gama_{tt}^w + \VV_t.
\end{equation}
Setting $t=0$ we have
\begin{equation}
 \depth\gama_{tt} + 2\depth_t\gama_t +\depth_{tt}\gama = 
(\skewm\OO^2 + \skewm{[\OO_t]})\depth\gama + 2\skewm\OO\Gama_t^w +
\Gama_{tt}^w + \VV_t.
\end{equation}

Now the expression for $\depth_{tt}$ in the theorem can be obtained by dotting with $\e_3$,
giving Equation~\ref{eq:depthtt:moving}. Isolating $\gama_{tt}$ we have
\begin{equation}\label{eq:gamatt:moving:compact}
\gama_{tt} = (\skewm\OO^2 + \skewm{[\OO_t]})\gama +
\frac{1}{\depth}(2\skewm\OO\Gama_t^w + \Gama_{tt}^w +
\VV_t - 2\depth_t\gama_t - \depth_{tt}\gama).
\end{equation}
Substituting Equations~\ref{eq:depthtt:moving}
and~\ref{eq:gama:time:deriv:Gamaw:moving} into the above, we obtain the final
expression for $\gama_{tt}$.
\end{proof}

\noindent\textbf{The Special Case of Fixed Points.}
%-------Moving - done. -------------------------------------------------------------------
The question of how the image of a fixed point moves as the camera moves was
studied by Longuet-Higgins and Prazdny~\cite{LonguetHiggins:Prazdny:PRSL:1980}
and later by Waxman and Ullman~\cite{Waxman:Ullman:IJRobotics85},
giving the velocity $\gama_t$ for a fixed point. This calculation also
leads to the well-known epipolar constraint, the notion of
Essential
matrix~\cite{LonguetHiggins:Essential:Nature:1980}, and 
the \emph{continuous epipolar
constraint}~\cite{Zhuang:Haralick:CAIA1984,Maybank:book:1992,Kanatani:IJCV1993,
Vieville:Faugeras:ICCV1995,Tian:etal:CVPR1996,Brooks:etal:JOSA1997,Astrom:Heyden:IJCV1998,Ponce:Genc:IJCV2998,Ma:etal:ECCV98,Ma:Soatto:etal:book,Stewenius:etal:CVPR2007,Lin:etal:IJCV2009,Valgaerts:etal:IJCV2012,Schneevoigt:etal:PR2014}.
Theorem~\ref{th:derivatives:allcontours} in the special case of a fixed point
gives interesting geometric insight into these classical results. Specifically,
setting $\Gama^w_t = 0$ in first-order computations of
Equations~\ref{eq:depth:t:Gamaw}
and~\ref{eq:gama:time:deriv:Gamaw:moving} results in
\begin{empheq}[left=\empheqlbrace,right=\qquad\text{ at $t=0$}]{align}
\label{eq:fixed:point:flow:depth}
\frac{\depth_t}{\depth} &= \e_3^\top(\skewm\OO\gama + \frac{\VV}{\depth})\\
\label{eq:fixed:point:flow:vectorial}
\gama_t &=  \OO_\times\gama -
(\e_3^\top\OO_\times\gama)\,\gama + \frac{\VV}{\depth} -
\frac{V_z}{\depth}\gama.
%\shortintertext{or}
%\frac{\depth_t}{\depth} = \e_3^\top(\skewm\OO\gama + \frac{\VV}{\depth}).
\end{empheq}

\vspace{1em}\noindent\textbf{Essential constraint.}
%\draftnote{Ben suggested relation btween aperture problem and continuous
%epipolar constraint, since continuous epipolar
%constraint gives a line where $\gama_t$ should lie, but not its tangential
%component.}
To derive the differential epipolar constraint,
eliminate $\depth$ from Equation~\ref{eq:fixed:point:flow:vectorial}
by first writing out the expression in terms of $\uu$, $\vv$, $u$, and $v$, where $\gama =
[\uu,\,\vv,\,1]^\top$, and $\gama_t = [u,\,v,\,0]$, and use $\e_3^\top\skewm\OO\gama =
-\Omega_y\uu + \Omega_x\vv$, giving
%\begin{empheq}[left={(\text{at }t=0)\,\,\,\empheqlbrace}]{align}
\begin{empheq}[left={\empheqlbrace}]{align}
\label{eq:lambda:from:udot}
u - \Omega_y\uu^2 + \Omega_x\uu\vv + \Omega_z\vv
-\Omega_y
= \frac{1}{\depth}(V_x - V_z\uu\nonumber)\\
 v +\Omega_x\vv^2 - \Omega_y\uu\vv - \Omega_z \uu +
\Omega_x
= \frac{1}{\depth}(V_y - V_z\vv)
\end{empheq}
and then eliminate $\depth$, giving
\begin{equation}\notag
\left(  u - \Omega_y\uu^2 + \Omega_x\uu\vv + \Omega_z\vv
-\Omega_y\right)
\left(V_y - \vv V_z\right) =
\left(  v +\Omega_x\vv^2 - \Omega_y\uu\vv - \Omega_z \uu +
\Omega_x \right)
\left(V_x - \uu V_z\right),
\end{equation}
which is the epipolar constraint for differential
motion. A more direct way of deriving the
epipolar constraint equation is to eliminate $\depth$ in
Equation~\ref{eq:gama:time:deriv:Gamaw:3d:2d:separated} with $\Gama^w_t = 0$ by
taking the dot-product with the vector $\VV\times\gama = \skewm\VV\gama$,
where $\skewm\VV$ is a skew-symmetric arrangement of $\VV$,%
\nomenclature[06100]{$\skewm{\VV}$}{(matrix, $m/s$)
Entries of $\VV$ arranged into a
skew-symmetric matrix such that $\skewm \VV \mathbf v = \VV\times\mathbf
v$ for any vector $\mathbf v$\nomrefpage}
and using $\skewm\OO^\top = -\skewm\OO$ gives
\begin{align}
\depth\gama_t^\top\skewm\VV\gama &=
-\depth\gama^\top\skewm\OO\skewm\VV\gama\nonumber,\\
\shortintertext{resulting in the \emph{differential epipolar constraint}}
\gama_t^\top\skewm\VV\gama &+
\gama^\top\skewm\OO\skewm\VV\gama =
0.
\end{align}
In comparison, the widely-known essential
constraint for relating two views is given by
\begin{equation}
\gama_2^\top\skewm\transl \rot\gama_1 = 0,
\end{equation}
where $\skewm\transl \rot$, the essential matrix, 
combines the effects of translation and rotation to relate
two points $\gama_1$ and $\gama_2$.
In the differential case, the two matrices $\skewm \VV$ and
$\skewm\OO\skewm\VV$ play a similar role to $\skewm\transl \rot$ in the discrete
motion case to relate a point and its velocity.

\begin{remark}
Observe from Equation~\ref{eq:fixed:point:flow:vectorial} that $\gama_t$ can also be written as the sum of two
components, one depending on $\VV$, and the other on $\OO$, \ie,
\begin{equation}\label{eq:star1square1}
\gama_t = \frac{1}{\depth}A(\gama)\VV + B(\gama)\OO,
\ \ \text{where}\ \ 
A(\gama) = 
\begin{bmatrix}
1 & 0 & -\uu\\
0 & 1 & -\vv\\
0 & 0 & 0
\end{bmatrix} \text{ and }
B(\gama) = 
\begin{bmatrix}
-\uu\vv & 1+\uu^2 & -\vv\\
-(1+\vv) & \uu\vv & \uu\\
0 & 0 & 0
\end{bmatrix}.
\end{equation}
That $\gama_t$ depends linearly on $\VV$
and $\OO$ (since $A$ and $B$ are only dependent on the position $\gama$),
with $\depth$ left in the equation, is the
basis of subspace methods in structure from motion~\cite{Jepson:Heeger:IJCV92}.
Observations of image velocities $\gama_{t,1},\gama_{t,2},\dots,\gama_{t,N}$ at
points $\gama_1, \gama_2, \dots, \gama_N$ 
provides $2N$ \emph{linear} equations in $\VV$ and $\OO$, given $\depth_1,\dots,\depth_N$.
\end{remark}

%% file: flow-deforming-curves.tex
\subsection{Differential Relations for a Curve}

\begin{theorem} \label{th:deforming:curve:diff} (Deforming 3D curve) Consider a deforming 3D curve $\Gama(s,t)$
projecting to a family of 2D curves $\gama(s,t)$ with depth $\depth(s,t)$,
arising from camera motion with differential velocities of translation and
rotation $\VV$ and $\OO$, respectively, and let $\VV_t$ and $\OO_t$ be
their respective derivatives in time.
Then, the image velocity $\gama_t$ is determined from
$\left\{\VV,\, \OO,\, \frac{\VV}{\depth}, \,\frac{\Gama^w_t}{\depth},\,
\t \right\}$,
\begin{empheq}[left=\text{$\gama_t = \alpha\t + \beta \n$, where \ }\empheqlbrace]{align}
  \alpha &=- \OO\cdot\gama\times(\gama\times\n)
  -\left(\frac{\VV}{\depth} -\skewm\OO\frac{\transl}{\depth}+
   \rot\frac{\Gama^w_t}{\depth}\right)\cdot\gama\times\n ,\label{eq:tangential:velocity:deforming:anyt}\\
  \beta  &= \OO\cdot\gama\times(\gama\times\t) +
  \left(\frac{\VV}{\depth}  -\skewm\OO\frac{\transl}{\depth} + 
  \rot\frac{\Gama^w_t}{\depth}\right)\cdot\gama\times\t.\label{eq:normal:velocity:deforming:anyt}
\end{empheq}
% todo uncomment for thesis.
%and at $t=0$
%\begin{empheq}[left=\empheqlbrace]{align}
%  \alpha &=- \OO\cdot\gama\times(\gama\times\n)
%  -\gama\times\n\cdot(\frac{\VV}{\depth} + \frac{\Gama^w_t}{\depth}),\label{eq:tangential:velocity:deforming}\\
%  
%  \beta  &= \OO\cdot\gama\times(\gama\times\t) +
%  \gama\times\t\cdot(\frac{\VV}{\depth} +
%  \frac{\Gama^w_t}{\depth}).\label{eq:normal:velocity:deforming}
%\end{empheq}
\end{theorem}
\begin{proof}
From Equation~\ref{eq:3D:point:velocity:allt} and using $-\rot\bc_t = \VV -
\skewm\OO\transl$ from Equation~\ref{eq:vv:def},
\begin{equation}
\Gama_t = \skewm\OO\Gama + \VV - \skewm\OO\transl + \rot\Gama^w_t.
\end{equation}
Using $\Gama = \depth\gama$ and $\Gama_t = \depth_t\gama +
\depth\gama_t$, 
\begin{equation}\label{eq:proof:alpha:beta:basic:eq}
\depth_t\gama + \depth\gama_t = \depth\skewm\OO\gama + \VV -
\skewm\OO\transl + \rot\Gama^w_t.
\end{equation}
Taking the dot product with $\gama\times\n$ and $\gama\times\t$,
\begin{equation}
\left\{
\begin{aligned}
\depth\gama_t\cdot(\gama\times\n) &=
\depth(\skewm\OO\gama)\cdot(\gama\times\n)
+ (\VV - \skewm\OO\transl)\cdot(\gama\times\n) +
\rot\Gama^w_t\cdot(\gama\times\n),\\
\depth\gama_t\cdot(\gama\times\t) &=
\depth(\skewm\OO\gama)\cdot(\gama\times\t)
+ (\VV - \skewm\OO\transl)\cdot(\gama\times\t)
+ \rot\Gama^w_t\cdot(\gama\times\t).
\end{aligned}\right.
\end{equation}
Now, 
\begin{equation}
\left\{
\begin{aligned}
\gama_t\cdot(\gama\times\n) &= (\alpha\t+\beta\n)\cdot(\gama\times\n) =
\alpha\t\cdot(\gama\times\n) = \alpha\n\times\t\cdot\gama = -\alpha\e_3^\top\gama =
-\alpha\\
\gama_t\cdot(\gama\times\t) &= (\alpha\t+\beta\n)\cdot(\gama\times\t) =
\beta\n\cdot(\gama\times\t) = \beta\t\times\n\cdot\gama = \beta\e_3^\top\gama =
\beta.
\end{aligned}\right.
\end{equation}
So that we can write
\begin{equation}
\left\{\begin{aligned}
\alpha &= -(\skewm\OO\gama)\cdot(\gama\times\n) -\left(
\frac{\VV}{\depth} - \skewm\OO\frac{\transl}{\depth}
+ \rot\frac{\Gama_t^w}{\depth}\right)\cdot(\gama\times\n) \\
\beta &= (\skewm\OO\gama)\cdot(\gama\times\t) +\left(
\frac{\VV}{\depth} - \skewm\OO\frac{\transl}{\depth}
+ \rot\frac{\Gama_t^w}{\depth}
\right)\cdot(\gama\times\t).
\end{aligned}\right.
\end{equation}
Since we can switch the cross and
dot products in a triple scalar product, $\OO\times\gama\cdot (\gama\times\n) = 
\OO\cdot\gama\times(\gama\times\n)$ and $\OO\times\gama\cdot (\gama\times\t) = 
\OO\cdot\gama\times(\gama\times\t)$, giving the final result.

\end{proof}
\begin{corollary}
The spatial variation of the velocity vector field $\gama_t$ along the curve and
in time can be written as
\begin{align}\label{eq:spatio:motion:derivative:deforming}
\gama_{st} &= \left(-\VV + V_z\gama\right)\frac{\depth_s}{\depth^2}
- \frac{V_z}{\depth}\gama_s + \skewm\OO\gama_s -
(\e_3^\top\skewm\OO\gama_s)\gama - (\e_3^\top\skewm\OO\gama)\gama_s\notag\\
& \frac{1}{\depth}( \Gama^w_{st} - \e_3^\top\Gama^w_{st}\gama - \e_3^\top\Gama^w_t\gama_s
) - \frac{1}{\depth^2}(\Gama^w_t - \e_3^\top\Gama^w_t\gama)\depth_s,
\end{align}
and the time acceleration $\gama_{tt}$ is defined by
%
%\begin{empheq}[left=\empheqlbrace]{align}\label{eq:gama:time:secondderiv:Gamaw:moving:frenet}
\begin{equation}\label{eq:gama:time:secondderiv:Gamaw:moving:frenet}
\left\{\begin{aligned}
\t^\top\gama_{tt} &= \t^\top(\skewm\OO^2 + \skewm{[\OO_t]})\gama +
\frac{2}{\depth}\t^\top\skewm\OO\Gama^w_t + \frac{1}{\depth}\t^\top\Gama^w_{tt} +
\frac{1}{\depth}\t^\top\VV_t \\
&- 2\e_3^\top\left(\skewm\OO\gama +
\frac{\VV}{\depth} + \frac{\Gama^w_t}{\depth}\right)
\,\alpha
- 
\e_3^\top\left( (\skewm\OO^2 + \skewm{[\OO_t]})\gama +
\frac{\VV_t}{\depth} + 2\skewm\OO\frac{\Gama^w_t}{\depth} + 
\frac{\Gama^w_{tt}}{\depth} \right)
\t^\top\gama,\\
\n^\top\gama_{tt} &= \n^\top(\skewm\OO^2 + \skewm{[\OO_t]})\gama +
\frac{2}{\depth}\n^\top\skewm\OO\Gama^w_t + \frac{1}{\depth}\n^\top\Gama^w_{tt} +
\frac{1}{\depth}\n^\top\VV_t \\
&- 2\e_3^\top\left(\skewm\OO\gama +
\frac{\VV}{\depth} + \frac{\Gama^w_t}{\depth}\right)
\,\beta
- 
\e_3^\top\left( (\skewm\OO^2 + \skewm{[\OO_t]})\gama +
\frac{\VV_t}{\depth} + 2\skewm\OO\frac{\Gama^w_t}{\depth} + 
\frac{\Gama^w_{tt}}{\depth} \right)
\n^\top\gama.
\end{aligned}\right.
\end{equation}
\end{corollary}
\begin{proof}
The $\gama_{st}$ expression
in~\eqref{eq:spatio:motion:derivative:deforming} is derived by differentiating $\gama_t$
with respect to $s$ in Equation~\ref{eq:gama:time:deriv:Gamaw:moving}. 
Notice that $\gama_t$ in the moving case decomposes into the same terms
as for the fixed case, Equation~\ref{eq:fixed:point:flow:vectorial},  plus
terms dependent on $\Gama^w_t$ given by
$\frac{1}{\depth} \left( \Gama^w_t - \e_3^\top\Gama^w_t\gama \right)$.
Differentiating with respect to $s$ then gives a term
equal to $\gama_{st}$ for the fixed case plus terms dependent on $\Gama^w_t$ and
its spatial derivative, the latter being obtained by differentiating the above
expression with respect to $s$.

The expressions of $\gama_{tt}$ in the Frenet frame were obtained by taking the
dot product of~\eqref{eq:gama:time:secondderiv:Gamaw:moving} with $\t$ and $\n$, noting that $\gama_t\cdot\t =
\alpha$ and $\gama_t\cdot\n = \beta$. We then plug in 
expressions~\eqref{eq:depth:t:Gamaw} and~\eqref{eq:depthtt:moving} for
$\depth_t$ and $\depth_{tt}$, respectively.
\end{proof}

\noindent\textbf{Special Case: Rigid Stationary Curve.}
\begin{corollary}(Rigid stationary 3D curve) Let $\Gama(\tilde s)$ be a 3D curve
projecting to a family of 2D curves $\gama(\tilde s,t)$ with depth
$\depth(\tilde s,t)$,
arising from camera motion with differential velocity of translation and
rotation $\VV$ and $\OO$, respectively. Let $\t$ denote the unit
tangent to the image curve. Then
\begin{equation}
\gama_{\tilde st} = \frac{-\depth_{\tilde s}}{\depth}\left( \frac{\VV}{\depth} -
\frac{V_z}{\depth}\gama \right) - \frac{V_z}{\depth}\t + \skewm\OO\t -
(\e_3^\top\skewm\OO\t)\gama - (\e_3^\top\skewm\OO\gama)\t.
\end{equation}
\end{corollary}
\begin{proof}
Follows by setting $\Gama^w_t = 0$ in
Equation~\ref{eq:spatio:motion:derivative:deforming} and using the
spatial parameter as the arc-length of the image curve.
\end{proof}
\begin{comment} % todo uncomment for thesis
, where the formulas appear with $\depth$ on the
left hand side,
\begin{empheq}[left=\empheqlbrace]{align}
  \depth\left[\alpha+ \OO\cdot\gama\times(\gama\times\n)\right]&=
  -\gama\times\n\cdot(\VV -
  \skewm\OO\transl)\label{eq:tangent:velocity:anyt:faugerasform}\\
  %
  \depth\left[\beta - \OO\cdot\gama\times(\gama\times\t)\right]&=
  \gama\times\t\cdot(\VV -
  \skewm\OO\transl)\label{eq:normal:velocity:anyt:faugerasform}
\end{empheq}
or, at $t = 0$,
\begin{empheq}[left=\empheqlbrace]{align}
  \depth\left[ \alpha + \OO\cdot\gama\times(\gama\times\n) \right] &=
  -\gama\times\n\cdot\VV.\label{eq:tangential:velocity:faugerasform}\\
  %
  \depth\left[ \beta - \OO\cdot\gama\times(\gama\times\t) \right] &=
  \gama\times\t\cdot\VV \label{eq:normal:velocity:faugerasform}
\end{empheq}
\end{comment}

\begin{corollary}\label{thm:tangential:normal:velocity:rigid} The tangential and
normal velocities of a rigid curve induced by a moving camera are derived from 
$\{\gama,\,\t,\,\n,\, \frac{\transl}{\depth},\, \OO,\, \frac{\VV}{\depth}\}$
for any $t$ as
\begin{empheq}[left=\empheqlbrace,right=\text{\qquad for any $t$,}]{align}
  \alpha &= -\OO\cdot\gama\times(\gama\times\n)
  -\left(\frac{\VV}{\depth} -
  \skewm\OO\frac{\transl}{\depth}\right)\cdot\gama\times\n\label{eq:tangent:velocity:anyt}\\
  \beta &= \OO\cdot\gama\times(\gama\times\t) +
  \left(\frac{\VV}{\depth} -
  \skewm\OO\frac{\transl}{\depth}\right)\cdot\gama\times\t\label{eq:normal:velocity:anyt}
\end{empheq}
or
\begin{empheq}[left=\empheqlbrace,right=\text{\qquad for $t=0$.}]{align}
  \alpha &=- \OO\cdot\gama\times(\gama\times\n)
  -\gama\times\n\cdot\frac{\VV}{\depth}\label{eq:tangential:velocity}\\
  \beta  &= \OO\cdot\gama\times(\gama\times\t) +
  \gama\times\t\cdot\frac{\VV}{\depth}\label{eq:normal:velocity}
\end{empheq}
\end{corollary}
\begin{proof}
  Follows directly from Theorem~\ref{th:deforming:curve:diff} and $\Gama^w_t = 0$.
\end{proof}
\begin{corollary}
The infinitesimal Essential constraint in the Frenet frame of the image of a
rigid curve is given by
\begin{equation}\label{eq:epipolar:constr:differential:intrinsic}
  (\gama\times\t)\cdot\VV\left[ \alpha + \OO\cdot\gama\times(\gama\times\n) \right]
  +(\gama\times\n)\cdot\VV\left[ \beta - \OO\cdot\gama\times(\gama\times\t)
  \right] = 0.
\end{equation}
\end{corollary}
\begin{proof}
Eliminate $\depth$ from~\eqref{eq:normal:velocity}
and~\eqref{eq:tangential:velocity}.
\end{proof}
\begin{corollary}(From~\cite{Papadopoulo:Faugeras:ECCV96,Papadopoulo:PhD:96})
The tangential velocity $\alpha$ can be fully determined
from the normal velocity $\beta$ and $\gama$, $\t$, $\n$, $\OO$,
and $\frac{\VV}{\depth}$ without the explicit knowledge of $\depth$,
as
\begin{equation}\label{eq:tangential:velocity:from:normal:velocity}
\alpha = -\left[ \beta - \OO\cdot\gama\times(\gama\times\t)
\right]\frac{\VV\cdot(\gama\times\n)}{\VV\cdot(\gama\times\t)} -
\OO\cdot\gama\times(\gama\times\n).
\end{equation}
\end{corollary}
\begin{proof}
Follows by solving~\eqref{eq:epipolar:constr:differential:intrinsic} for $\alpha$.
\end{proof}

\noindent\textbf{Special Case: Occluding Contours.}
A remarkable observation is derived below that the first-order deformation of an apparent contour under
epipolar parametrization does not depend on the 3D surface geometry, since the
curvature-dependent terms cancel out for an occluding contour,
\cf~\cite{Giblin:Motion:Book}.

\begin{theorem} \label{thm:occl:contours:flow} (Occluding contours) Let $\Gama(s,t)$ be the contour generator
for apparent contours $\gama(s,t)$.  Then the image velocity $\gama_t$ at $t=0$
can be determined from $\gama$ by $\depth$ and the
infinitesimal motion parameters using
Equation~\ref{eq:fixed:point:flow:vectorial}, \ie, the same one used for 
a stationary contour.
\end{theorem}
\begin{proof}
Recall from Equation~\ref{eq:epipolar:param:eq} that the velocity of an
occluding contour under epipolar parametrization statisfies $\Gama^w_t =
\lambda (\Gama^w - \bc)$ for some $\lambda$, so that at $t=0$,
\begin{equation}
\Gama^w_t = \lambda\depth\gama \implies \e_3^\top\Gama_t^w = \lambda\depth,
\end{equation}
so that $\Gama^w_t = (\e_3^\top\Gama^w_t)\gama$ and the terms
$\Gama^w_t - (\e_3^\top\Gama^w_t)\gama = 0$
so that all appearances of $\Gama^w_t$ cancel-out
altogether in Equation~\ref{eq:gama:time:deriv:Gamaw:moving}, giving exactly the
same formula as for fixed contours,
Equation~\ref{eq:fixed:point:flow:vectorial}, when $\Gama^w_t = 0$.
\end{proof}

We now show exactly how the velocity of the 3D occluding contour,
$\Gama^w_t$, depends on the curvature of the occluding
surface~\cite{Cipolla:PHD:1991,Cipolla:Blake:IJCV92}.

\begin{theorem}\label{th:3D:velocity:occluding}
The velocity of a 3D occluding contour under epipolar parametrization and relative to a fixed world coordinate
system (camera at $t=0$) is given by
\begin{empheq}[left=\empheqlbrace]{align}
\Gama_t^w &= -\frac{\bc_t^\top\N^w}{K^t}\cdot\frac{\Gama^w - \bc}{\|\Gama^w -
\bc\|^2}, & & \text{ for arbitrary $t$.}\\
\Gama_t^w &= -\frac{\bc_t^\top\N}{K^t}\cdot\frac{\gama}{\depth\|\gama\|^2}, &  &\text{
for $t = 0$,}
\end{empheq}
or, in terms of $\transl$ and $\rot$, and image measurements,
\begin{equation}
\Gama_t^w = \frac{1}{K^t} \left(\frac{\VV^\top}{\depth}\,
\frac{\gama\times\t}{\|\gama\times\t\|}\right)\frac{\gama}{\|\gama\|^2},\ \ \ \ \ \text{
for $t = 0$,}\label{eq:3D:velocity:occluding:img}
\end{equation}
where $K^t$ is the normal curvature of the occluding surface along the visual direction.
\end{theorem}
\begin{proof}
The desired \emph{formulae} can be consistently derived by adapting
variant~\cite{Astrom:Cipolla:Giblin:IJCV1999} of the original result
by Cipolla and Blake~\cite{Cipolla:PHD:1991,Cipolla:Blake:IJCV92} to the
proposed notation. This must be performed carefully to
establish correctness in a solid way.
We thus provide an alternative, clearer proof without using unit view spheres.

The normal curvature of the occluding surface along the visual direction is
given by classical differential geometry~\cite{Giblin:Motion:Book} as
\begin{equation}\label{eq:k:tt}
  K^t = - \frac{\Gama^{w\top}_t \N^w_t}{\Gama^{w\top}_t\Gama^w_t},
\end{equation}
using epipolar parametrization.
Substituting the epipolar parametrization condition of the second form
of~\eqref{eq:epipolar:param:eq},
\begin{equation}
  K^t = - \frac{(\Gama^w - \bc)^\top\N^w_t}{\lambda \|\Gama^w - \bc\|^2}.
\end{equation}
Isolating $\lambda$ and plugging back into the epipolar parametrization
condition,
\begin{equation}\label{Gamat:occl:intermetidiate}
\Gama_t^w = -\frac{(\Gama^w-\bc)^\top\N^w_t}{K^t}\,\frac{\Gama^w - \bc}{\|\Gama^w -
\bc\|^2}.
\end{equation}
We now show that $(\Gama^w-\bc)^\top\N^w_t = -\bc_t^\top\N^w$, thereby arriving
at the desired expression for $\Gama_t^w$. In fact,
differentiating the occluding contour condition in the second form of
Equation~\ref{eq:occlusion:condition} gives
\begin{align}
(\Gama_t^w - \bc_t)^\top\N^w + (\Gama^w - \bc)^\top\N_t^w &= 0,\\
-\bc_t^\top\N^w + (\Gama^w - \bc)^\top\N_t^w &= 0
\end{align}
which, together with~\eqref{Gamat:occl:intermetidiate} produces the desired
result
\begin{equation}
\Gama_t^w = -\frac{\bc_t^\top\N^w}{K^t}\,\frac{\Gama^w - \bc}{\|\Gama^w -
\bc\|^2},\ \ \ \ \ \text{ for arbitrary $t$.}
\end{equation}
At $t=0$, we have $\N^w = \N$ and $\Gama^w - \bc = \Gama = \depth\gama$ (but note that
$\Gama_t(0)\neq \Gama^w_t(0)$), hence
\begin{align}
\Gama_t^w &= -\frac{\bc_t^\top\N}{K^t}\,\frac{\gama}{\depth\|\gama\|^2},\ \ \ \ \ \text{
for $t = 0$.}
\end{align}%
Using $\VV = -\bc_t$ from Equation~\ref{eq:vv:def} and
$\N = \frac{\gama\times\t}{\|\gama\times\t\|}$ gives
the alternative form of this equation.
\end{proof}

We now present a theorem relating observed quantities to camera motion, which is
key for calibrating 3D motion models from families of projected deforming contours
observed in video
sequences with unknown camera motion, among other applications. A form
of this theorem appears
in~\cite{Papadopoulo:Faugeras:ECCV96,Papadopoulo:PhD:96}, Equation L1, but this
is limited to rigid motion. The following theorem generalizes the results to include
occluding contours.
The fact that Equation~\ref{eq:papadopoulo:l1} in the theorem is also valid for
occluding contours is a new result, to the best of our
knowledge. The term $\Gama_t^w$ is zero for fixed contours, and is dependent on
surface curvature in the case of occluding contours. The equation is not valid
for arbitrary non-rigid contours because, in order to derive the normal flow
equation, we used $\Gama^w_t\cdot(\gama\times\t) = 0$, which is only true for
occluding and fixed contours.

\begin{theorem}(A generalized form of the L1 equation of
\cite{Papadopoulo:Faugeras:ECCV96,Papadopoulo:PhD:96})
Given a 3D \textbf{occluding contour or fixed curve}, and the family of projected curves
$\gama(t)$ observed in a monocular sequence of images from a moving
camera, and given $\t,\kappa,\n,\beta,\beta_t$ measurements at one point, then the
first and second order camera motion, $\OO$, $\VV$, $\OO_t$, $\VV_t$
satisfy the polynomial equation
\begin{equation}
\begin{aligned}\label{eq:papadopoulo:l1}
&V_z\,[\beta - \OO\cdot\gama\times(\gama\times\t)]^2 + 
\VV\cdot\gama\times\t\left(\beta_t - \OO_t\cdot\gama\times(\gama\times\t) -
\OO\cdot\left[ \gama\times(\gama\times\t) \right]_t\right) \\
&- \left[ \VV_t\cdot\gama\times\t + \VV\cdot(\gama\times\t)_t \right][\beta - \OO\cdot\gama\times(\gama\times\t)] 
+ \VV\cdot\gama\times\t\,(\e_3\cdot\skewm\OO\gama)[\beta -
\OO\cdot\gama\times(\gama\times\t)]\\
&+ \e_3\cdot\Gama_t^w[\beta - \OO\cdot\gama\times(\gama\times\t)]^2 +
(\skewm\OO\VV)(\gama\times\t)[\beta - \OO\cdot\gama\times(\gama\times\t)]
= 0.
\end{aligned}
\end{equation}
\end{theorem}
\begin{proof}
The normal velocity $\beta$ of an image contour follows 
Equation~\eqref{eq:normal:velocity:anyt}, which holds for both stationary
curves, Corollary~\ref{thm:tangential:normal:velocity:rigid}, and for occluding
contours, Theorem~\ref{thm:occl:contours:flow}.  Differentiating it with respect to time,
\begin{align}
\depth_t\beta + \beta_t\depth = &\depth_t\OO\cdot\gama\times(\gama\times\t)
+ \depth\OO_t\cdot\gama\times(\gama\times\t) + \depth\OO
[\gama\times(\gama\times\t)]_t\\
&+ (\gama\times\t)_t(\VV -\skewm\OO\transl)
+ (\gama\times\t) (\VV_t - \OO_t\times\transl - \skewm\OO\VV)\notag
\end{align}
Rearraging the terms,\\[-0.8em]
\begin{align}
&\depth_t[\beta - \OO\cdot\gama\times(\gama\times\t)] + \depth [\beta_t - \OO_t\cdot\gama\times(\gama\times\t) -
\OO\cdot[\gama\times(\gama\times\t)]_t]\notag\\
&= (\gama\times\t)_t(\VV - \skewm\OO\transl) + (\gama\times\t) (\VV_t -
\OO_t\times\transl - \skewm\OO\VV).
\shortintertext{Setting $t=0$,}
&\depth_t[\beta - \OO\cdot\gama\times(\gama\times\t)] + \depth [\beta_t - \OO_t\cdot\gama\times(\gama\times\t) -
\OO\cdot [\gama\times(\gama\times\t)]_t]\notag\\
&= (\gama\times\t)_t\VV + (\gama\times\t) (\VV_t -\skewm\OO\VV)
\end{align}
Now, from Equation~\ref{eq:depth:t:Gamaw}, we can plug-in an expression for $\depth_t$ at
$t=0$, 
\begin{align}
&(\depth\e_3^\top\skewm\OO\gama + V_z + \e_3^\top\Gama^w_t)[\beta - \OO\cdot\gama\times(\gama\times\t)] + \depth (\beta_t - \OO_t\cdot\gama\times(\gama\times\t) -
\OO\cdot [\gama\times(\gama\times\t)]_t)\notag\\
&= (\gama\times\t)_t\VV + (\gama\times\t) (\VV_t -\skewm\OO\VV),
\end{align}
which is analogous to Equation 7.28 of~\cite[p.167]{Papadopoulo:PhD:96}, but
this time with occluding contours also
being included. Now, eliminating depth $\depth$
using Equation~\ref{eq:normal:velocity}, \eg, by multiplying the above by
$[\beta - \OO\cdot\gama\times(\gama\times\t)]$, we obtain
\begin{equation}
\begin{aligned}
&[V_z(\beta - \OO\cdot\gama\times(\gama\times\t)) +
(\VV\cdot\gama\times\t)\,\e_3\cdot\skewm\OO\gama +\\
&\e_3\cdot\Gama_t^w(\beta - \OO\cdot\gama\times(\gama\times\t))]\,[\beta -
\OO\cdot\gama\times(\gama\times\t)] \\
&+ \VV\cdot\gama\times\t\left[ \beta_t -
\OO_t\cdot\gama\times(\gama\times\t) - \OO\cdot
[\gama\times(\gama\times\t)]_t \right] =\\
&[\VV_t\cdot\gama\times\t + \VV(\gama\times\t)_t -
(\skewm\OO\VV)(\gama\times\t)]\,[\beta -
\OO\cdot\gama\times(\gama\times\t)].
\end{aligned}
\end{equation}
Rearranging the terms, we obtain the desired equation.
\end{proof}

\begin{remark}
Note that previously reported results for the rigid case~\cite[eq.\ 7.12]{Papadopoulo:PhD:96} have
an apparently missing term corresponding to the last term in our Equation~\ref{eq:papadopoulo:l1}, 
\begin{equation}\notag
(\skewm\OO\VV)(\gama\times\t)[\beta - \OO\cdot\gama\times(\gama\times\t)].
\end{equation}
This is due to the fact that they used slightly different variables for the
translational component of the infinitesimal motion equations, but the results
are mathematically the same for the rigid case.
\end{remark}

%\indraftnote{Ric: Equation~\ref{eq:papadopoulo:l1} has secretly embedded
%tangential velocity! $(\gama\times\t)_t$ needs $\t_t$, which in turn requires
%knowledge of tangential velocity of $\gama_t$. $\t_t = \frac{1}{g}(\alpha\kappa +
%\beta_s)\n$.}

% todo thesis
%\noindent\textbf{On the role of curves.} Note that
%Equation~\ref{eq:papadopoulo:l1} was mostly
%obtained by differentiating the motion and projection equations twice with
%respect to time. Hence the contribution of curves for these equations is next to
%none -- all
%they do is provide a direction (the normal) along which disparity can be
%measured.  It seems that the equation would work flawlessly if we were to
%deal with only a single fixed 3D point, and the tangent direction $\t$ taken as
%the $x$ axis versor $\e_1$. The only reason the equation would need a spatial
%derivative (tangent $\t$) is to calculate the normal to the surface in the case of
%an occluding contour.

%\draftnote{derived on february 5, 2008.}
\begin{theorem}\label{th:spatio:motion:derivative:fixed}
The first spatial derivative of image apparent motion of both a fixed curve
and an occluding contour under epipolar correspondence is given by
\begin{equation}\label{eq:spatio:motion:derivative:fixed}
\gama_{st} = \left(-\frac{\VV}{\depth} + \frac{V_z}{\depth}\gama\right)\frac{\depth_s}{\depth}
- \frac{V_z}{\depth}\gama_s + \skewm\OO\gama_s -
(\e_3^\top\skewm\OO\gama_s)\gama - (\e_3^\top\skewm\OO\gama)\gama_s.
\end{equation}
Note that the 
derivative of depth $\depth_s$ can be expressed in terms of 3D
curve geometry as $\depth_s = \e_3^\top\Gama_s$.
% todo uncomment for thesis
%\begin{equation}\label{eq:spatial:depth:derivative}
%\depth_s = \e_3^\top\Gama_s
%\end{equation}
\end{theorem}
\begin{proof}
Equation~\ref{eq:spatio:motion:derivative:fixed} follows by differentiating the fixed flow~\eqref{eq:fixed:point:flow:vectorial} 
with respect to $s$, observing
that only $\depth$ and $\gama$ depend on $s$. The formula for $\depth_s$ is
obtained from the observation that the dot product of $\Gama = \depth\gama$ with
$\e_3$ gives $\e_3^\top\Gama =
\depth$. Differentiating this with respect to $s$ gives $\depth_s = \e_3^\top\Gama_s$.
\end{proof}

\begin{comment} % todo uncomment for thesis.
\paragraph{Remark about Degrees of Freedom}
The actual number of \dof\
in~\eqref{eq:spatio:motion:derivative:fixed} of 
Theorem~\ref{th:spatio:motion:derivative:fixed} might be smaller than the 8-
parametric representation. For example, depth always appear as dividing the
velocity $\VV$. Moreover, although~\eqref{eq:spatial:depth:derivative} involves many
parameters, it actually has only 1\dof\ as well since the dot product
$\gama^\top\skewm\OO\Gama^w_s$ is a scalar that can be treated as an unknown.
\end{comment}

Theorem~\ref{th:derivatives:allcontours} gives an expression for the
image acceleration of a moving 3D point, which includes points lying on any type
of contour (even non-rigid), in terms of the evolution of the 3D curve. Since the
latter is expressed in terms of a fixed world coordinate system, the motion of
the object and the motion of the cameras are written down separately, even
though they exert joint effects on image velocity.

\begin{theorem}
\label{th:image:acceleration:occluding}
The image acceleration of an \textbf{occluding contour} under epipolar
parametrization is given by
\begin{equation}
\begin{aligned}
\gama_{tt} &= (\skewm\OO^2 + \skewm{[\OO_t]})\gama 
- [\e_3^\top(\skewm\OO^2 + \skewm{[\OO_t]})\gama]\gama
+ 2\skewm\OO\frac{\Gama^w_t}{\depth}
+ \frac{\VV_t}{\depth} - \frac{2\depth_t}{\depth}\gama_t
+ \frac{\e_3^\top\Gama^w_t}{\depth}\gama_t \\
&-\frac{\e_3^\top\Gama^w_t}{\depth}\skewm\OO\gama 
- \frac{\e_3^\top\VV_t}{\depth}\gama
- \frac{2\e_3^\top\skewm\OO\Gama^w_t}{\depth}\gama
% new term discovered during second ijcv revision:
+\frac{(e_3^\top\Gama^w_t)(\e_3^\top\skewm\OO\gama)}{\depth}\gama
\qquad\qquad\text{at $t=0$,}
\end{aligned}
\end{equation}
where $\gama_t$ and $\depth_t$ are given by
Equations~\ref{eq:fixed:point:flow:vectorial}
and~\ref{eq:fixed:point:flow:depth}, and
$\Gama^w_t$ is dependent on curvature,
Equation~\ref{eq:3D:velocity:occluding:img}.
\end{theorem}
\begin{proof}
Substituting
Equation~\ref{eq:depthtt:moving} into Equation~\ref{eq:gamatt:moving:compact}, we
get
{\small \begin{equation}\label{eq:gamatt:moving:substituted:occl}
\gama_{tt} = (\skewm\OO^2 + \skewm{[\OO_t]})\gama +
\frac{2\skewm\OO\Gama_t^w}{\depth}+ \frac{\Gama^w_{tt}}{\depth} +
\frac{\VV_t}{\depth} - \frac{2\depth_t\gama_t}{\depth} - [
\e_3^\top(\skewm\OO^2 + \skewm{[\OO_t]})\gama]\gama
-\frac{\e_3^\top\VV_t}{\depth}\gama -
\frac{2\e_3^\top\skewm\OO\Gama_t^w}{\depth}\gama -
\frac{\e_3^\top\Gama_{tt}^w}{\depth}\gama.
\end{equation}}%
Now, let $\boldv$ be the viewing direction in world coordinates, so that
\begin{equation}
\gama = \rot\boldv,
\end{equation}
and let $\ff$ be the normal to the image plane in world coordinates, so that
\begin{equation}
\e_3 = \rot\ff.
\end{equation}
Thus, 
\begin{equation}
\e_3^\top\gama = \ff^\top \rot^\top \rot \boldv = \ff^\top\boldv = 1.
\end{equation}
Note also that at $t=0$ we have $\ff = \e_3 $ and $\gama = \boldv$.
Now, the condition for epipolar parametrization of an occluding contour,
Equation~\ref{eq:epipolar:param:eq}, can be expressed as
\begin{equation}\label{eq:Gamatw:epipolar}
\Gama^w_t = \lambda \boldv,
\end{equation}
for some scalar factor $\lambda$. Taking the dot product with $\ff$ we have
\begin{empheq}[left=\empheqlbrace]{align}\label{eq:lambda:epipolar}
\lambda &= \ff^\top\Gama^w_t,\\
\Gama_t^w &= \ff^\top\Gama_t^w\boldv.
\end{empheq}
Differentiating~\eqref{eq:Gamatw:epipolar} with respect to time and using~\eqref{eq:lambda:epipolar} gives
\begin{align}
\Gama_{tt}^w = \lambda_t\boldv + \lambda\boldv_t &= \lambda_t\boldv +
\ff^\top\Gama_t^w\boldv_t.
\shortintertext{Taking the dot product with $\ff$,}
\ff^\top\lambda_t\boldv + \ff^\top(\ff^\top\Gama_t^w)\boldv_t&= \ff^\top\Gama^w_{tt}\\
\lambda_t &= \ff^\top\Gama^w_{tt} - (\ff^\top\Gama_t^w)\ff^\top\boldv_t.
\shortintertext{Thus,}
\Gama^w_{tt} = (\ff^\top\Gama^w_{tt})\boldv  -
(\ff^\top&\Gama^w_t)(\ff^\top\boldv_t)\boldv + (\ff^\top\Gama^w_t)\boldv_t,\\
\label{eq:gama:tt:epipolar:condition:temp}
\Gama^w_{tt} = \e_3^\top\Gama^w_{tt}\gama -
(\e_3^\top&\Gama^w_t)(\e_3^\top\boldv_t)\gama + \e_3^\top\Gama_t^w\boldv_t
& \text{at $t=0$}.
\end{align}
In order to get $\boldv_t(0)$ in terms of $\gama$ we write
\begin{align}
\gama_t &= \rot_t\boldv + \rot\boldv_t.
\shortintertext{Thus}
\boldv_t &= \gama_t - \skewm\OO\gama\ \ \ \ \ \ \ \ \text{at $t=0$.}
\shortintertext{Substituting back into~\eqref{eq:gama:tt:epipolar:condition:temp},}
\Gama^w_{tt} = \e_3^\top\Gama^w_{tt}\gama + \e_3^\top\Gama^w_t\gama_t &-
\e_3^\top\Gama_t^w\skewm\OO\gama
+(\e_3^\top\Gama^w_t)(\e_3^\top\skewm\OO\gama)\gama & \text{at $t=0$.}
\end{align}
Plugging this equation onto~\eqref{eq:gamatt:moving:substituted:occl},
the $\e_3^\top\Gama^w_{tt}\gama/\depth$ terms cancel out, giving the final equation.
\end{proof}